\providecommand{\pgfsyspdfmark}[3]{}
\documentclass[twoside,11pt]{article}

\usepackage{blindtext}
\usepackage[preprint]{jmlr2e}
\usepackage[utf8]{inputenc}
\usepackage{amsmath, amsfonts, bm, amssymb, mathtools, bm, commath}

\usepackage{algorithmic,algorithm}
\usepackage{enumerate}

\renewcommand{\cal}[1]{\mathcal{#1}}

\renewcommand{\r}{\mathbb{R}}

\newcommand{\n}{\mathbb{N}}

\newcommand{\mmag}[1]{\left|#1\right|}

\newcommand{\Ebb}[1]{\mathbb{E}\left[#1\right]}

\newcommand{\iprod}[2]{\left\langle{#1},{#2}\right\rangle}
\usepackage[symbol]{footmisc}
\newcommand{\Pacd}{\cal{P}_2^{1}(\r^{d_x})}
\newcommand{\Macd}{\cal{M}_2^{1}}
\newtheorem{assumption}{Assumption} 

\DeclareMathOperator*{\argmax}{arg\,max}
\DeclareMathOperator*{\argmin}{arg\,min}

\usepackage[textwidth=1.8cm, textsize=scriptsize]{todonotes}
\usepackage{lastpage}

\jmlrheading{26}{2025}{1-\pageref{LastPage}}{3/24; Revised 2/25}{5/25}{24-0437}{Rocco Caprio, Juan Kuntz, Samuel Power and Adam M. Johansen}

\ShortHeadings{Convergence of PGD}{Caprio, Kuntz, Power and Johansen}
\firstpageno{1}

\begin{document}

\title{Error Bounds for Particle Gradient Descent,\\ and Extensions of the log-Sobolev and Talagrand Inequalities}

\author{\name  Rocco Caprio \email rocco.caprio@warwick.ac.uk \\
       \addr Department of Statistics\\
       University of Warwick\\
       Coventry, CV4 7AL, UK
       \AND
       \name Juan Kuntz 
       \AND
       \name Samuel Power \email sam.power@bristol.ac.uk \\
       \addr School of Mathematics \\
       University of Bristol\\
       Bristol, BS8 1UG, UK
       \AND
       \name Adam M. Johansen \email a.m.johansen@warwick.ac.uk \\
       \addr Department of Statistics\\
       University of Warwick\\ 
       Coventry, CV4 7AL, UK}

\editor{Alexandre Bouchard}

\maketitle

\begin{abstract}%
We derive non-asymptotic error bounds for particle gradient descent (PGD, \citet{PGD}), a recently introduced algorithm for maximum likelihood estimation of large latent variable models obtained by discretizing a gradient flow of the free energy. 
We begin by showing that the flow converges exponentially fast to the free energy's minimizers for models satisfying a condition that generalizes both the log-Sobolev and the Polyak--Łojasiewicz inequalities (LSI and PŁI, respectively). We achieve this by extending a result well-known in the optimal transport literature (that the LSI implies the Talagrand inequality) and its counterpart in the optimization literature (that the PŁI implies the so-called quadratic growth condition), and applying the extension to our new setting. We also generalize the Bakry--\'Emery Theorem and show that the LSI/PŁI  extension holds for models with strongly concave log-likelihoods. For such models, we further control PGD's discretization error and obtain the non-asymptotic error bounds. While we are motivated by the study of PGD, we believe that the inequalities and results we extend may be of independent interest.
\end{abstract}


	\section{Introduction}
	
	Many tasks in machine learning and statistics require fitting a probabilistic model---with Lebesgue density, $p_\theta(x,y)$, and featuring latent variables, $x$---to data, $y$, we have observed. Often, we achieve this by finding model parameters, $\theta$, that maximize the probability, $p_\theta(y)$, of observing the data we observed (the \emph{marginal likelihood}). That is, $\theta_*$ belonging to
	\begin{equation} \label{eq:MLE}
		\cal{O}_*:= \argmax_{\theta\in\r^{d_\theta}} p_\theta(y)= \argmax_{\theta\in \r^{d_\theta}} \int p_\theta(x,y)\dif x
	\end{equation}
	(we assume throughout that the latent variables and parameters respectively take values in $\r^{d_\theta}$ and $\r^{d_x}$). In many cases, the latent variables are meaningful or interesting in some way, and we would like to infer them. Following the empirical Bayes paradigm~\citep{Robbins1955}, we do this using the posterior distribution of the latent variables given the data for the optimal parameters $\theta_*$:
	$$p_{\theta_*}(x|y):=\frac{p_{\theta_*}(x,y)}{p_{\theta_*}(y)}.$$
	For most models of practical interest, the integral in \eqref{eq:MLE} is intractable, we have no closed-form expressions for $p_\theta(y)$ or its derivatives, and we are unable to directly optimize $p_\theta(y)$. Often, this hurdle is overcome by noting that $( \theta_*, q_*)$ minimizes the \emph{free energy} functional,
	\begin{equation}
		\label{eq:freeenergy}
		F( \theta, q):=\left\{\begin{array}{cl} 
			\displaystyle\int \log\left(\frac{q(x)}{p_\theta(x,y)}\right) q(\dif x)&\text{if } q\ll \dif x\\ +\infty&\text{otherwise}\end{array}\right. \quad \forall (\theta, q) \in \cal{M},
	\end{equation}
	if and only if $\theta_*$ maximizes the marginal likelihood and $q_* = p_{\theta_*} (\cdot \mid y)$; where $\dif x$ in \eqref{eq:freeenergy} denotes the Lebesgue measure  and $\cal{M}$ the product of the parameter space, $\r^{d_\theta}$, and the space, $\cal{P}(\r^{d_x})$, of probability distributions over the latent space $\r^{d_x}$. For instance, the well-known expectation-maximization (EM) algorithm~\citep{Dempster1977} can be viewed as minimizing $F$ using coordinate descent (CD)~\citep{Neal1998}. Similarly, many methods in variational inference (approximately) minimize $F$ by (i) restricting $\cal{P}(\r^{d_x})$ to a parametric family of distributions $(q_\phi)_{\phi \in \Phi}$ such that the surrogate objective $( \theta, \phi) \mapsto F (\theta, q_\phi)$ is tractable, (ii) computing a minimizer $(\theta_*, \phi_*)$ thereof
	using an appropriate optimization algorithm, and (iii) employing $(\theta_*, q_{\phi_*})$ as a proxy for a genuine minimizer of $F$; e.g., see~\citet{Kingma2019}.
	
	Inspired by the interpretation of the EM algorithm as CD applied to $F$, \citet{PGD} examines whether it is possible to minimize $F$ over $\cal{M}$, rather than any restriction thereof, using analogues of optimization algorithms other than CD. In particular, they identify analogues of gradient descent (GD) applicable to $F$. To do so, they recall that the application of GD to minimizing a function $f:\r^{d_\theta}\to\r$ may be viewed as the forward Euler discretization of the ordinary differential equation (ODE) known as the \emph{gradient flow}:
	\begin{equation}\label{eq:gradient_flow_euclidean}
		\dot{\theta}_t = -\nabla_\theta f(\theta_t),
	\end{equation}
	where $\nabla_\theta$ denotes the usual Euclidean gradient on $\r^{d_\theta}$. Borrowing ideas from optimal transport, they identify $\nabla F = (\nabla_\theta F, \nabla_q F)$, with
	\begin{equation}\label{eq:gradF}
		\nabla_\theta F( \theta, q) = - \int \nabla_\theta \ell(\theta,x) q(\dif x),\quad 
		\nabla_q F( \theta, q ) = \nabla_x \cdot \left[ q \nabla_x \log\left(\frac{p_{\theta}(\cdot, y)}{q}\right)\right],
	\end{equation}
	and $\ell(\theta, x):=\log(p_\theta(x,y))$, as an analogue to $\nabla_\theta f$ and obtain the following gradient flow dynamics for the parameter estimate $\theta_t$ and the latent posterior approximation $q_t$:
	\begin{align}\label{eq:pde}
		\dot{\theta}_t = \int \nabla_\theta \ell(\theta_t,x)q_t(\dif x),\qquad
		\dot{q}_t=\nabla_x\cdot\bigg[q_t \nabla_x\log\bigg(\frac{q_t}{p_{\theta_t}(\cdot|y)}\bigg)\bigg].
	\end{align}
	They go on to (i) interpret~\eqref{eq:pde} as the Fokker--Planck equation for the following McKean--Vlasov stochastic differential equation (SDE; e.g., see~\citet{Chaintron2022}):
	\begin{align}\label{eq:sde1}
		\dif \theta_t = \int \nabla_\theta \ell( \theta_t, x ) q_t(\dif x)\dif t,\quad
		\dif X_t = \nabla_x \ell( \theta_t, X_t) \dif t + \sqrt{2} \dif W_t, \hspace{0.3em}\text{where} \enskip q_t:=\textrm{Law}(X_t);
	\end{align}
	and (ii) to approximate the law $q_t$ with an empirical distribution $N^{-1} \sum_{n=1}^N \delta_{X^n}$ of \emph{particles} $\smash{X^1, \ldots, X^N}$; and (iii) discretize the resulting SDE in time. The result is a practical algorithm for fitting latent variable models they refer to as particle gradient descent (PGD, Algorithm~\ref{alg:pgd}; see also \citet{Wang2025} for a version suited to latent diffusion models):
	\begin{algorithm}[H]
		\begin{algorithmic}[1]
			\STATE{\textbf{Inputs:} step size $h$, step number $K$, particle number $N$, and initial particles $X^{1,h}_0,\dots,$ $X_0^{N,h}$ and parameter $\Theta_0^{N,h}$.\hspace{-1pt}}
			\FOR{$k=0,\dots, K-1$}
			\STATE{Update the parameter estimate:\vspace{-16pt}
				\begin{equation}\label{eq:pgd_update_theta}\Theta_{k+1}^{N,h} = \Theta_{k}^{N,h} + \frac{h}{N}\sum_{n=1}^N \nabla_\theta  \ell(\Theta_k^{N,h},X_k^{n,h}).\end{equation}\vspace{-12pt}}
			\STATE{Update the particles: with $W_k^1,\dots,W_k^N$ denoting i.i.d.\ $\cal{N}(0,I_{d_x})$ r.v.s,
				\begin{equation}
					X_{k+1}^{n,h}=X_{k}^{n,h}+h\nabla_x \ell(\Theta_k^{N,h},X_{k}^{n,h})+\sqrt{2h}W_k^n\quad \forall n=1,\dots,N.\label{eq:pgd_update_x}
				\end{equation}\vspace{-24pt}
			}
			\ENDFOR
			\RETURN{$\Theta_K^{N,h}$ and $Q^{N,h}_K:=N^{-1}\sum_{n=1}^N\delta_{X_K^{n,h}}$.}
		\end{algorithmic}
		\caption{Particle gradient descent (PGD).}
		\label{alg:pgd}
	\end{algorithm}
	PGD performs well in experiments~\citep{PGD} and is well-suited to modern computing environments. Its updates~(\ref{eq:pgd_update_theta},\ref{eq:pgd_update_x}) require only evaluating $\nabla \ell := ( \nabla_\theta \ell, \nabla_x \ell)$ and adding noise, thus facilitating its implementation in modern autodiff frameworks such as TensorFlow, PyTorch, and JAX. Moreover, each particle's update~\eqref{eq:pgd_update_x} is independent of all other particles, offering scope for simple parallelisation and distribution of these steps. Implemented naively, a full run of the algorithm costs $\cal{O}(N K [$evaluation cost of $\nabla\ell])$ operations, where $N$ denotes the number of particles employed and $K$ the total number of steps taken. The running time's dependence on $N$ can be further reduced through the aforementioned parallelism strategies.
	
	Here, we  validate theoretically PGD and show that, for models whose log-likelihoods $\ell$ are $\lambda$-strongly concave with $L$-Lipschitz gradients, its output, $(\Theta_K^{N,h},Q^{N,h}_K)$, satisfies 
	\begin{equation}\label{eq:errorbigO}
		\mathsf{d} (( \Theta_K^{N,h}, Q^{N,h}_K ), ( \theta_*, Q_*^N )) = \mathcal{O} (h^\frac{1}{2} + N^{-\frac{1}{2}} + e^{-h\lambda K}),
	\end{equation}
	assuming that the algorithm's step size $h$ is no greater than $1/(\lambda+L)$. In the above, $Q_*^N$ denotes the empirical distribution of $N$ i.i.d.~particles drawn from $p_{\theta_*}(\cdot|y)$ and $\mathsf{d}$ the following metric  acting on $\cal{M}$-valued random variables:
	\begin{equation}\label{eq:dmetric}
		\mathsf{d} (( \Theta, Q ),(\Theta',Q')) := \sqrt{\Ebb{\mathsf{d}_{E} (\Theta,\Theta')^2} + \Ebb{\mathsf{d}_{W_2} (Q,Q')^2}}
	\end{equation}
	where $\mathsf{d}_E$ and $\mathsf{d}_{W_2}$ respectively denote the Euclidean and Wasserstein-2 metrics.
	
	We begin by studying PGD's continuous-time infinite-particle limit: the gradient flow~\eqref{eq:pde}. First, we show that the free energy $(F(\theta_t,q_t))_{t\geq0}$ evaluated along the flow converges at an exponential rate to its infimum, 
	$$F_*:=\inf_{(\theta,q)\in\cal{M}}F(\theta,q),$$ 
	provided that the model satisfies a condition which generalizes the log-Sobolev inequality (LSI) popular in optimal transport (e.g., see \citealp[Chapter 21]{Villani2009}) and the Polyak--Łojasiewicz inequality (PŁI) used in optimization \citep{Karimi2016}. In particular, we assume that there exists a constant $\lambda > 0$ such that 
	\begin{equation}\label{eq:extlogsobolev}
		2 \lambda [F (\theta, q) - F_*] \leq I( \theta, q ) 
	\end{equation}
	for all parameters $\theta$ and distributions $q$, where the functional $I$ is defined by
	\begin{align} 
		\label{eq:fisherinformation}
		I (\theta, q) :=& \norm{\int \nabla_\theta \ell ( \theta, x ) q(\dif x)}^2 + \int \norm{\nabla_x \log \left( \frac{q(x)}{p_{\theta} (x,y)}\right)}^2 q(\dif x),
	\end{align}
	with $\norm{\cdot}$ denoting the Euclidean norm. We then extend a result of \citet{Otto2000} and show that models which satisfy~\eqref{eq:extlogsobolev} also satisfy the following generalization of both the Talagrand inequality~\citep{Talagrand1996} and the \emph{quadratic growth} condition used in optimization~\citep{Anitescu2000,Karimi2016}: for all parameters $\theta$ and distributions $q$, it holds that
	\begin{equation} \label{eq:exttalagrand}
		2 [F (\theta, q) - F_*] \geq \lambda \mathsf{d} (( \theta, q ),\cal{M}_*)^2 ,
	\end{equation}
	where $\cal{M}_*$ denotes $F$'s optimal set and $ \mathsf{d} (( \theta, q ),\cal{M}_*)$ the distance from $(\theta,q)$ to $\cal{M}_*$:
	\begin{equation}\label{eq:Mstar}
		\cal{M}_* := \argmin_{( \theta, q ) \in \cal{M}} F( \theta, q ),\quad   \mathsf{d} (( \theta, q ),\cal{M}_*):=\inf_{ (\theta_*, q_*) \in \cal{M_*}} \mathsf{d}( (\theta, q), (\theta_* ,q_*)).
	\end{equation}
	Using~\eqref{eq:exttalagrand}, we then show that the flow itself $(\theta_t, q_t)_{t\geq0}$ converges to $\cal{M}_*$ exponentially fast in $\mathsf{d}$.
	Next, we generalize the Bakry--{\'E}mery Theorem and show that models with strongly concave likelihoods $\ell$ satisfy \eqref{eq:extlogsobolev}. For such models, we obtain a slightly stronger convergence result. Under the further assumption that $\ell$'s gradients are Lipschitz, we are able to bound the errors introduced by both approximating distributions with empirical measures and discretizing~\eqref{eq:pde} in time, and the bound \eqref{eq:errorbigO} follows.
	\subsection{Related Literature}
		The analysis provided in \citet{PGD} was limited: as explained in the `Our setting, notation, assumptions, rigour, and lack thereof' paragraph of the paper's introduction, it focused on intuition and practical application rather than theoretical validation.  While the authors obtained no error bounds, they did argue with their Theorem~3 that the gradient flow~\eqref{eq:pde} converges exponentially fast to $F$'s minimizer whenever $\ell$ is strongly concave. However, they imposed an unnecessary technical condition that substantially restricts the applicability of their result (that the gradient of $\ell$ be bounded uniformly). 
	
	Recently, several works have proposed particle-based algorithms for maximum marginal likelihood estimation. \citet{Akyildiz2023} consider an alternative approximation to the flow~\eqref{eq:pde} which they term the `Interacting Particle Langevin Algorithm' or IPLA; see also \citet{Encinar2024} for a version applicable to non-differentiable models. IPLA differs from PGD by the inclusion of an extra additive Gaussian noise term $\sqrt{2h/N}W_k$ in the parameter's update~\eqref{eq:pgd_update_theta}. This difference enables an analysis of the algorithm using tools similar to those used to study the unadjusted Langevin algorithm in \citet{Durmus2017,Durmus2019}. In particular, the authors show that the $\cal{L}^2$ error of the parameter estimates produced by IPLA satisfies (\ref{eq:errorbigO},RHS) for models satisfying the same Lipschitz gradient and strong concavity assumptions made here, which we find pleasantly consistent with our results in Section \ref{sec:pgdconvergence}.  \citet{Lim2023,Chen2023,Oliva2024} consider variants of PGD and IPLA that incorporate momentum in the updates and can outperform PGD and IPLA in practice. These variants no longer approximate the gradient flow~\eqref{eq:pde} and, consequently, require a different type of analysis.  In particular, \citet{Lim2023} established exponentially fast convergence of their algorithm's limiting (non-gradient) flow under \eqref{eq:extlogsobolev}, while \citet{Oliva2024} obtained non-asymptotic error bounds for their algorithms' parameter estimates also under Lipschitz gradient and strong concavity assumptions.  \citet{Sharrock2023} introduced two new particle-based methods for minimizing $F$. The first, `Stein Variational Gradient Descent EM' or SVGD EM, is also constructed as an approximation to a gradient flow of $F$, but with respect to a different underlying geometry---the so-called `Stein Geometry' \citep{Liu2016,Duncan2023}, as opposed to the Euclidean/Wasserstein-2 geometry underpinning~\eqref{eq:gradF} (whose induced distance metric is given by~\eqref{eq:dmetric} omitting the expectations; cf. \citealp[Appendix A]{PGD} for more on this). The second, `Coin EM', involves an entirely different approach for minimizing $F$ that builds on coin betting techniques from convex optimization and obviates the need for tuning discretization step sizes. Only SVGD-EM is analyzed theoretically. To do so, the authors assume a natural `Stein' analogue of \eqref{eq:extlogsobolev}: Assumption 7 in \citet[Appendix~A]{Sharrock2023}.
	\citet{Fan2023} considers a non-parametric maximum likelihood estimation problem where they optimize a measure over the parameters rather than the parameters themselves by approximating $F$'s gradient flow in the Fisher--Rao/Wasserstein-2 geometry. Finally, \citet{Crucinio2025} recently proposed an approach for maximum marginal likelihood estimation based on mirror descent and sequential Monte Carlo methods to minimize $F$.
 
	Our results in Section~\ref{sec:flowconvergence} relate to earlier results well-known in the optimal transport and optimization literatures. We discuss these connections as we go along.
	\subsection{Paper Structure}
	In Section \ref{sec:flowconvergence}, we study the convergence of the gradient flow \eqref{eq:pde} and the inequalities~(\ref{eq:extlogsobolev},\ref{eq:exttalagrand}). In Section \ref{sec:pgdconvergence}, we analyze the flow's approximation (PGD; Algorithm~\ref{alg:pgd}) and obtain non-asymptotic bounds on its error. We conclude in Section~\ref{sec:conclusion} by discussing our results beyond the context of PGD.
	\section{Convergence of the Gradient Flow}\label{sec:flowconvergence}
	
	In this section, we study the gradient flow~\eqref{eq:pde} and several pertinent inequalities. As we explain in what follows, these inequalities generalize others well-known in the literature, and we prefix the former with `extended' to differentiate them from the latter. We start by showing that for models satisfying the extended LSI (xLSI)~\eqref{eq:extlogsobolev}, the values of $F$ along the flow converge exponentially fast to the infimum of $F$ (Section \ref{sec:lsi}). We next argue that any model satisfying the xLSI also satisfies the extended Talagrand-type inequality~\eqref{eq:exttalagrand}, and we use this result to show that the trajectory of the flow itself converges exponentially fast to $(\theta_*, p_{\theta_*}(\cdot|y))$ for some maximizer $\theta_*$ of the marginal likelihood (Section~\ref{sec:talagrand}). Finally, we show that the xLSI holds for models with $\lambda$-strongly concave log-likelihoods $\ell$, and we obtain a slightly stronger convergence result for this case (Section~\ref{sec:bakry-emery}). 
	
	\paragraph{Notation and Assumptions.} For the remainder of the paper, we write $\rho_\theta(\cdot)$ for the likelihood $p_\theta( \cdot, y)$, $Z_\theta$ for the marginal likelihood $p_\theta(y)$, and $\pi_\theta(\cdot)$ for the posterior distribution $p_\theta( \cdot \mid y)$. To ensure that  $\mathsf{d}$ in~\eqref{eq:dmetric} is a metric, we consider the restriction of $\cal{P}(\r^{d_x})$ to the subset of  measures with finite second moments, and similarly for $\cal{M}$:
	$$\cal{P}_2(\r^{d_x}):=\bigg\{q\in\cal{P}(\r^{d_x}):\int \norm{x}^2 q(\dif x) <\infty\bigg\},\quad\cal{M}_2:=\r^{d_\theta}\times\cal{P}_2(\r^{d_x}).$$
	To ensure that the  functional $I$ in~\eqref{eq:fisherinformation} is well-defined, we further restrict restrict $\cal{P}_2(\r^{d_x})$ to the subset of measures with densities differentiable almost everywhere w.r.t.~to the Lebesgue measure $\dif x$, and similarly for $\cal{M}_2$: 
	\begin{align*}
		\Pacd &:= \bigg\{q \in \cal{P}_2(\r^{d_x}) : q\ll \dif x,\,  \nabla_x \frac{\dif q}{\dif x}(x) \text{ exists a.e.} \bigg\},\quad \cal{M}_2^1:=\r^{d_\theta}\times\Pacd.
	\end{align*}
	We also assume that the model is suitably differentiable:
	\begin{assumption}[Model regularity]\label{ass:model} (i) For all $x$ in $\r^{d_{x}}$, $\theta \mapsto \pi_\theta (x)$ is differentiable; and $\theta \mapsto Z_\theta$ is differentiable; (ii) for all $\theta$ in  $\r^{d_\theta}$, $\pi_\theta$ is twice continuously differentiable;  (iii) for all $\theta$ in $\r^{d_{\theta}}$ and $x$ in $\r^{d_x}$, $\rho_\theta(x)>0$;  
	\end{assumption}
	Next, we assume that the gradient flow~\eqref{eq:pde} has sufficiently regular solutions. To this end, let $\cal{C}^i(\cal{X},\cal{Y})$ denote the space of $i$-times continuously differentiable functions from $\cal{X}$ to $\cal{Y}$, $\cal{C}^i_c(\cal{X},\cal{Y})$ the subspace of such functions with compact support, and $\cal{C}^{i,j}(\cal{X}\times\cal{X}',\cal{Y})$ denote the functions that are $\cal{C}^i(\cal{X},\cal{Y})$ in the first variable and $\cal{C}^j(\cal{X}',\cal{Y})$ in the second. 
	\begin{assumption}[Regularity of solutions]\label{ass:solutions_existence} 
		For any initial conditions $(\theta, q)$ in $\cal{M}_2$, \eqref{eq:pde} has a classical solution $(\theta_t, q_t)_{t\geq0}$ with $(\theta_0, q_0)=(\theta,q)$. For any such solution, $q_t(dx)$ has a Lebesgue density $q_t(x)$ for all $t> 0$; $(t,x)\mapsto q_t(x)$ belongs to $\cal{C}^{1,2}([0,\infty)\times\r^{d_x},(0,\infty))$, and $t\mapsto\theta_t$ belongs to $\cal{C}^1([0,\infty),\r^{d_\theta})$.
	\end{assumption}
	We verify in Section~\ref{sec:pgdconvergence} that the above holds whenever the model's log-likelihood $\ell$ has a Lipschitz gradient. We believe it will also hold under weaker conditions on $\ell$ because, in the special case that $\pi_\theta$ is a single distribution $\pi$ independent of $\theta$, the flow~\eqref{eq:pde} essentially reduces to the following well-known Fokker--Planck equation,
	\begin{equation}\label{eq:fpe}\dot{q}_t=\nabla_x\cdot\bigg[q_t \nabla_x\log\bigg(\frac{q_t}{\pi}\bigg)\bigg];\end{equation}
    and stronger regularity properties for \eqref{eq:fpe} have been established under weaker conditions on $\ell$. For example, \citet[Theorem 5.1]{JKO1998} shows that the solutions are smooth whenever $\log(\pi)$ is smooth. 
	
	%
	%
	\subsection{The Extended log-Sobolev Inequality and $(F(\theta_t,q_t))_{t\geq0}$'s Convergence}\label{sec:lsi}
	We prove the flow's convergence for models which satisfy the xLSI in the following sense: 
	\begin{definition}[Extended log-Sobolev inequality (xLSI)]\label{def:LSI} We say that the measures $(\rho_\theta (\dif x))_{\theta \in \r^{d_\theta}}$ satisfy the extended log-Sobolev inequality (xLSI) with constant $\lambda > 0$ if \eqref{eq:extlogsobolev} holds for all $(\theta, q)$ in $\Macd$.
	\end{definition}
	\noindent 
	Under this assumption, the values of $F$ converge along the gradient flow exponentially fast to $F$'s infimum over $\cal{M}_2$:
	\begin{theorem}[xLSI $\Rightarrow$ exp.~conv.~of $(F(\theta_t,q_t))_{t\geq0}$] \label{thm:flowconvergence} 
		If Assumptions~\ref{ass:model}--\ref{ass:solutions_existence} hold, $(\theta_0,q_0)$ belongs to $\cal{M}_2$, and the measures $(\rho_\theta (\dif x))_{\theta \in \mathbb{R}^{d_\theta}}$ satisfy the xLSI with constant $\lambda>0$, then
		\begin{equation} \label{eq:flowentropicconvergence}
			0 \leq  F (\theta_t, q_t) - F_* \leq [F( \theta_0, q_0 )-F_*] e^{-2\lambda t} \quad \forall t \geq 0.
		\end{equation}
	\end{theorem}
	
	\noindent {\bf Proof}{.
		The leftmost inequality in \eqref{eq:flowentropicconvergence} follows from the $F_*$'s definition. As we show in Lemma \ref{lemma:debruijin} in Appendix \ref{app:debruijn}, the following extension to de Bruijn's identity holds:
		\begin{equation} \label{eq:debruijn}
			\frac{\dif}{\dif t}F(\theta_t,q_t) = -I(\theta_t,q_t) \quad \forall t>0.
		\end{equation}
        Now the rightmost equality follows by combining the xLSI~\eqref{eq:extlogsobolev} with \eqref{eq:debruijn} and applying Gr\"onwall's inequality. 
	} \hfill\BlackBox \\
    \newline 
	As is well known,
\begin{equation}\label{eq:straightforwardcalcs1}F(\theta,q)=\mathrm{KL}(q||\pi_{\theta})-\log(Z_\theta)\quad\forall (\theta,q)\in\cal{M},\end{equation}
	where $\text{KL}$ denotes the Kullback--Leibler divergence:
	\begin{equation*}\mathrm{KL}(q||\pi) := \left\{\begin{array}{cl} 
			\displaystyle\int \log\left(\frac{q(x)}{\pi(x)}\right) q(\dif x)&\textup{if } q\ll \dif x\\ +\infty&\textup{otherwise}\end{array}\right.  \quad \forall q \in \cal{P}(\r^{d_x}),\end{equation*}
	Because $\mathrm{KL}(q||q')$ is non-negative and zero iff $q=q'$ a.e., \eqref{eq:straightforwardcalcs1}~implies that
	\begin{align}\label{eq:straightforwardcalcs2}
		F_* = \inf_{( \theta, q ) \in \cal{M}} F( \theta, q ) =\inf_{\theta\in\r^{d_\theta}}F(\theta,\pi_\theta)=-\log \left( \sup_{\theta \in \r^{d_\theta}} Z_\theta \right) =-\log Z_*,
	\end{align}
	where $Z_*:=\sup_{\theta\in\r^{d_\theta}}Z_\theta$ denotes the marginal likelihood's supremum. Putting (\ref{eq:flowentropicconvergence}--\ref{eq:straightforwardcalcs2}) together, we find that
	$$\log(Z_*)-\log(Z_{\theta_t})+\mathrm{KL}(q_t||\pi_{\theta_t})=F(\theta_t,q_t)+\log(Z_*)=\cal{O}(e^{-2\lambda t}).$$
	In other words, Theorem~\ref{thm:flowconvergence} shows that as $t$ increases: the free energy $F(\theta_t,q_t)$ converges exponentially fast to $-\log Z_*$, the log-marginal likelihood $\log(Z_{\theta_t})$ converges exponentially fast to its supremum $\log(Z_*)$, and $q_t$ tracks the corresponding posterior distributions $\pi_{\theta_t}$ in the sense that the $\mathrm{KL}$ divergence between the two decays exponentially fast to zero. \\

	\textit{Connections with the log-Sobolev inequality.} 
	We refer to the inequality in Definition~\ref{def:LSI} as the `xLSI' because it extends the classical log-Sobolev inequality (LSI) \citep{Gross1975}; e.g., see \citet[Definition~21.1]{Villani2009}. While the xLSI is a statement about a parametrized family of measures, the LSI is a statement about a single probability measure.
	In particular, if $\pi_\theta$ is a distribution $\pi$ independent of $\theta$, then~\eqref{eq:extlogsobolev} reads 
	\begin{equation} \label{eq:stdlogsobolev}
		2 \lambda \mathrm{KL} (q||\pi) \leq I (q||\pi) \enskip \forall q \in \Pacd;\enskip\text{where}\enskip I(q||\pi) := \int \norm{\nabla_x\log\left(\frac{q(x)}{\pi(x)}\right)}^2 q(\dif x).
	\end{equation}
	For this reason, Theorem~\ref{thm:flowconvergence} extends a well-known result (e.g., \citet[Section 5.2]{Bakry2014}) stating that, if $\pi$ satisfies the LSI~\eqref{eq:stdlogsobolev} and $(q_t)_{t\geq0}$ solves the Fokker--Planck equation~\eqref{eq:fpe}, then the $\mathrm{KL}$ divergence from $q_t$ to $\pi$ decays exponentially fast. \\

	\textit{Connections with the Polyak--Łojasiewicz inequality.} 
	Definition~\ref{def:LSI} also extends an inequality due to \citet{Polyak1963} and \citet{Lojasiewicz1963} commonly used to argue linear convergence of gradient descent algorithms. A differentiable real-valued function $f$ on $\r^{d_\theta}$ with infimum $f_*$ is said to satisfy the PŁI with constant $\lambda > 0$ if
	\begin{equation} \label{eq:PL}
		2\lambda [f(\theta)-f_*] \leq \norm{\nabla_\theta f(\theta)}^2\quad \theta\in \r^{d_\theta}.
	\end{equation}
	In particular, because
	\begin{equation}\label{eq:pl_der}
		\int \nabla_\theta \ell (\theta,x) \pi_\theta(\dif x) = \int \frac{\nabla_\theta \rho_\theta(x)}{\rho_\theta(x)} \pi_\theta(\dif x) = \frac{\nabla_\theta\int  \rho_\theta(x) \dif x}{Z_\theta} = \frac{\nabla_\theta Z_\theta}{Z_\theta} = \nabla_\theta \log(Z_\theta),
	\end{equation}
	setting $q:=\pi_\theta$ in \eqref{eq:extlogsobolev}, we recover~\eqref{eq:PL} with $f(\theta):=-\log(Z_\theta)$ (the exchange of limits in~\eqref{eq:pl_der} is valid whenever Fisher's identity holds, which is a common assumption in the study of latent variable models and the EM algorithm; it is satisfied whenever e.g. $\pi_\theta$ belongs to an exponential family, see~\citealt[Appendix D]{Douc2014}). In this case, Theorem~\ref{thm:flowconvergence} reduces to a well-known result showing that $f$  converges exponentially fast along the Euclidean gradient flow~\eqref{eq:gradient_flow_euclidean} to its infimum whenever $f$ satisfies the PŁI; e.g.,~see~\citet[Proposition~2.3]{trillos2023}.
	Among other reasons, the PŁI~\eqref{eq:PL} is popular in optimization because---without requiring the objective to be convex---it implies that the objective's stationary points are global minimizers. In our case, the xLSI implies that the marginal likelihood's stationary points are global maximizers. \\
	\subsection{The Talagrand Inequality and the Flow's Exponential Convergence}\label{sec:talagrand}
	We now show that, for models satisfying the xLSI, the gradient flow $(\theta_t,q_t)_{t\geq0}$ itself converges exponentially fast to the free energy's minimizers. We measure the convergence in terms of the following metric on $\cal{M}_2$:
	\begin{align}\label{eq:ddeterministic}
		\mathsf{d}(( \theta, q ),(\theta',q')):=\sqrt{\mathsf{d}_E(\theta,\theta')^2+\mathsf{d}_{W_2}(q,q')^2}.
	\end{align}
	Since $F$ is a lower semicontinuous function on $(\cal{M}_2,\mathsf{d})$ (see Lemma~\ref{lemma:Flsc} in Appendix~\ref{app:exttalagrand}), the set $\cal{M}_*$ in \eqref{eq:Mstar} of the free energy's minima is closed (note that Assumption~\ref{ass:model} and~\eqref{eq:straightforwardcalcs2} imply that $\cal{M}_*$ is contained in $\cal{M}_2$). Using (\ref{eq:straightforwardcalcs1},\ref{eq:straightforwardcalcs2}), we can characterize $\cal{M}_*$ in terms of the marginal likelihood's set of maximizers  ($\cal{O}_*$ in~\eqref{eq:MLE}):
	\begin{equation}\label{eq:optimalset}
		\cal{M}_* = \{(\theta_*,\pi_{\theta_*}) : \theta_* \in \cal{O}_*\}.
	\end{equation}
	To argue the flow's convergence, we require the following extension of the Talagrand inequality:
	\begin{definition}[Extension of the Talagrand inequality (xT$_2$I)]\label{def:talagrand} 
		We say that the measures $(\rho_\theta(\dif x))_{\theta\in\r^{d_\theta}}$ satisfy an extension of the Talagrand inequality (xT$_2$I) with constant $\lambda > 0$ if \eqref{eq:exttalagrand} holds for all $( \theta, q )$ in $\cal{M}_2$.
	\end{definition}
	\noindent The inequality holds for all models satisfying the xLSI:
	\begin{theorem}[xLSI $\Rightarrow$ xT$_2$I]\label{thm:exttalagrand}
		If Assumptions~\ref{ass:model}--\ref{ass:solutions_existence} hold, and the measures $(\rho_\theta(\dif x))_{\theta\in\r^{d_\theta}}$ satisfy the xLSI with constant $\lambda>0$, then they also satisfy the xT$_2$I with the same constant. 
	\end{theorem}
	\noindent See Appendix \ref{app:exttalagrand} for the proof. The convergence of $(\theta_t, q_t)_{t \geq 0}$ then follows from 
	Theorem~\ref{thm:flowconvergence}:
	\begin{corollary}[xLSI $\Rightarrow$ exp.~conv.~of $(\theta_t,q_t)_{t\geq0}$] \label{cor:dflowconvergence} 
		If Assumptions~\ref{ass:model}--\ref{ass:solutions_existence} hold, $(\theta_0,q_0)$ belongs to $\cal{M}_2$, and the measures $(\rho_\theta(\dif x))_{\theta\in\mathbb{R}^{d_\theta}}$ satisfy the xLSI with constant $\lambda>0$, then 
		\begin{equation*}
			\lambda \mathsf{d} ((\theta_t,q_t),\cal{M}_*)^2 \leq 2 [F(\theta_0,q_0)-F_*] e^{-2\lambda t}\quad\forall t\geq0.
		\end{equation*}
	\end{corollary}
	
	\textit{Connections with the Talagrand inequality.}
	By~\eqref{eq:optimalset}, if $\pi_\theta $ is a distribution $\pi$ independent of $\theta$, then $\cal{M}_* = \r^{d_\theta} \times \{ \pi \}$ and \eqref{eq:exttalagrand} reduces to the Talagrand inequality~\citep[Theorem 1.2]{Talagrand1996}:
	\begin{equation} \label{eq:stdtalagrand}
		2\mathrm{KL} (q||\pi) \geq \lambda \mathsf{d}_{W_2} (q,\pi)^2\quad\forall q\in\cal{P}_2(\r^{d_x}).
	\end{equation}
	For this reason, Theorem~\ref{thm:exttalagrand} extends the Otto--Villani Theorem \citep[Theorem 1]{Otto2000} showing  that the LSI~\eqref{eq:stdlogsobolev} implies~\eqref{eq:stdtalagrand}. \\

	\textit{Connections with the quadratic growth condition.}
	Setting $(q,f(\theta),f_*)$ to $( \pi_\theta,$ $-\log(Z_\theta),F_*)$, \eqref{eq:exttalagrand} reduces to the `quadratic growth' condition used in optimization to establish linear convergence rates to local minima of gradient descent algorithms:
	\begin{equation}  \label{eq:quadgrowthcondition}
		2 [f(\theta)-f_*] \geq \lambda \mathsf{d}_E (\theta, \cal{O}_*)^2 \quad \forall \theta \in \r^{d_\theta}.
	\end{equation}
	In this case, Theorem~\ref{thm:exttalagrand} reduces to the well-known result stating that~\eqref{eq:quadgrowthcondition} holds whenever the PŁI~\eqref{eq:PL} holds; see \citet[Theorem~2]{Karimi2016}. 
	\subsection{Strongly log-Concave Models}\label{sec:bakry-emery}
	In Section~\ref{sec:pgdconvergence}, we study PGD's convergence for models satisfying the following assumption.
	\begin{assumption}[Strong log-concavity]\label{ass:strongconcave} There exists a $\lambda>0$ such that 
		\begin{equation*} \label{eq:ellstrongconcave}
			\ell ( (1-t) \theta + t \theta', (1-t) x + t x') \geq (1 - t) \ell (\theta,x) + t \ell (\theta',x') + \frac{\lambda t(1-t)}{2} \norm{(\theta,x) - (\theta',x')}^2,
		\end{equation*}   
		for all $(\theta,x), (\theta',x')$ in $\mathbb{R}^{d_\theta} \times \mathbb{R}^{d_x}$ and $0 \leq t \leq 1$.
	\end{assumption}
	\noindent Models that satisfy the above also satisfy the xLSI:
	\begin{theorem}[Strong log-concavity $\Rightarrow$ xLSI]\label{thm:extbakryemery}Any model satisfying Assumptions~\ref{ass:model} and \ref{ass:strongconcave} satisfies the xLSI with constant $\lambda>0$.
	\end{theorem}
	\noindent {\bf Proof.}{ The result follows from the fact that the free energy is strongly geodesically convex if log-likelihood is strongly concave; see Appendix \ref{app:extbakryemery} for details.
	} \hfill\BlackBox \\

Simple examples satisfying Assumption~\ref{ass:strongconcave}, and consequently the xLSI, include Bayesian logistic regression (e.g., \citealp{Genkin2007}; see the proof of Proposition 1 in Appendix E.2 of \citealp{PGD} for an argument) and certain types of hierarchical models (\citealp{Gelman2007}; see \citealp[Example 1]{Caprio2024} for details). \\

	\textit{Connection with the Bakry--{\'E}mery Theorem.}
	Given that~\eqref{eq:extlogsobolev} reduces to the standard LSI~\eqref{eq:stdlogsobolev} whenever $\pi_\theta$ is independent of $\theta$, Theorem~\ref{thm:extbakryemery} extends Bakry--{\'E}mery Theorem~\citep{Bakry1985} showing that distributions with strongly log-concave densities satisfy~\eqref{eq:stdlogsobolev}. \\

	\textit{Connection with the optimization literature.}
	Given that~\eqref{eq:extlogsobolev} with $q = \pi_\theta$ reduces to the PŁI~\eqref{eq:PL} with $f(\theta) = -\log(Z_\theta)$, Theorem~\ref{thm:extbakryemery} extends the well-known fact that the PŁI holds for strongly convex $f$; e.g., see \citet[Theorem~2]{Karimi2016}. \\

	\textit{The flow's convergence.} Under Assumption~\ref{ass:strongconcave}, the marginal likelihood has a unique maximizer $\theta_*$; e.g., see~\citet[Theorem~4 in Appendix~B.3]{PGD}. Hence, Corollary~\ref{cor:dflowconvergence} and~\eqref{eq:optimalset} tell us that the flow converges exponentially fast to $(\theta_*,\pi_{\theta_*})$:
	\begin{align*}
		\lambda \mathsf{d} ((\theta_t,q_t),(\theta_*,\pi_{\theta_*}))^2 \leq 2 [F(\theta_0,q_0)-F_*]e^{-2\lambda t}\quad\forall t\geq0.
	\end{align*}
	In this special case, we can sharpen the above slightly (see Appendix~\ref{app:dflowconvproof} for the proof):
	\begin{theorem}[Strong log-concavity $\Rightarrow$ exp.~conv.~of $(\theta_t,q_t)_{t\geq0}$] \label{thm:dflowconvergencehp} If Assumptions~\ref{ass:model}(iii), \ref{ass:solutions_existence}, and \ref{ass:strongconcave} hold, and $(\theta_0,q_0)$ belongs to $\cal{M}_2$, then
		\begin{align*}
			\mathsf{d} ((\theta_t,q_t),(\theta_*,\pi_{\theta_*})) \leq \mathsf{d}((\theta_0,q_0),(\theta_*,\pi_{\theta_*}))e^{-\lambda t}\quad\forall t\geq0.
		\end{align*}
	\end{theorem}
	\section{Error Bounds for Particle Gradient Descent}\label{sec:pgdconvergence}
	Here, we capitalize on the results of Section~\ref{sec:flowconvergence} to obtain the error bound in \eqref{eq:errorbigO} for PGD (Algorithm~\ref{alg:pgd}) applied to models with strongly concave log-likelihoods. In particular, Theorem~\ref{thm:dflowconvergencehp} in Section~\ref{sec:bakry-emery} bounds the distance between the gradient flow and $(\theta_*,\pi_{\theta_*})$, where $\theta_*$ denotes the marginal likelihood's unique maximizer and the distance is measured in terms of the metric $\mathsf{d}$ in \eqref{eq:ddeterministic}. To exploit this bound and obtain~\eqref{eq:errorbigO}, we first need to connect the gradient flow to PGD, which we do via the McKean--Vlasov SDE~\eqref{eq:sde1}. To ensure that the SDE has globally-defined solutions, we assume that $\ell$'s gradient is Lipschitz:
	\begin{assumption} \label{ass:gradLip}The log-likelihood $\ell$ is differentiable and its gradient $\nabla \ell=(\nabla_\theta\ell,\nabla_x\ell)$ is $L$-Lipschitz for some $L>0$:
		$$\norm{\nabla \ell(\theta,x) - \nabla\ell(\theta',x')} \leq L \norm{(\theta,x) - (\theta',x')}\quad \forall (\theta,x),(\theta',x') \in \mathbb{R}^{d_\theta} \times \mathbb{R}^{d_x}.$$
	\end{assumption}
	\noindent
	Let $\cal{L}^2(\r^{d_x})$ denote the space of square-integrable random variables taking values in $\r^{d_x}$. The SDE's relation to the gradient flow is as follows (see Appendix \ref{app:sdesolutions} for the proof):
	\begin{proposition} \label{prop:sdesolutions}
		If Assumptions \ref{ass:model}(ii) and \ref{ass:gradLip} hold and $(\theta,X)$ belongs to $\r^{d_\theta}\times \cal{L}^2(\r^{d_x})$, then the SDE \eqref{eq:sde1} has an unique strong solution $(\theta_t,X_t)_{t\geq0}$ such that $(\theta_0,X_0)=(\theta,X)$. Moreover, $(\theta_t,\textup{Law}(X_t))_{t\geq0}$ is a classical solution of the gradient flow \eqref{eq:pde} satisfying Assumption~\ref{ass:solutions_existence}.
	\end{proposition}
	PGD, on the other hand, is obtained by approximating~\eqref{eq:sde1} as we detail in Section~\ref{sec:errrorboundproof} below. Bounding the errors introduced by these approximations, we obtain~\eqref{eq:errorbigO}. In particular, recall that $(\Theta_{K}^{N,h}, Q_{K}^{N,h})$ denotes PGD's output after $K$ iterations and $Q^N_*$ the empirical distribution of $N$ i.i.d.~particles drawn from $\pi_{\theta_*}$. Moreover, note that, the log-likelihood $\ell$ has a unique maximizer $(\theta_\dagger,x_\dagger)$ which, without loss of generality, we assume lies at the origin $(0,0)$. Our main result for PGD is then as follows (see Section~\ref{sec:errrorboundproof} for the proof):
	\begin{theorem}[PGD error bound]
		\label{thm:errorbound}Under Assumptions \ref{ass:model}(ii--iii) and \ref{ass:strongconcave}--\ref{ass:gradLip}, if $X_0^{1,h},\dots,X_0^{N,h}$ are drawn independently from a distribution $q_0$ in $\cal{P}_2(\r^{d_x})$ and $h \leq 1/(\lambda+L)$, then 
		\begin{align} \label{eq:pgderrorbound}
			\mathsf{d} ((\Theta_{K}^{N,h},Q_{K}^{N,h}), (\theta_*,Q_*^N)) \leq \sqrt{h} A_{0,h}+\frac{L\sqrt{2}}{\lambda \sqrt{N}} \sqrt{B_0 + \frac{2d_x}{\lambda}} + \mathsf{d}((\theta_0,q_0),(\theta_*,\pi_{\theta_*}))e^{-h\lambda K}
		\end{align}
		for all $K$ in $\n$; 
		where $B_0 :=  \norm{\theta_0}^2 + \Ebb{\norm{X_0}^{2}}$ and  
		\begin{flalign*}
		A_{0,h} := &\sqrt{\frac{4h+4/\iota}{\iota}220L^2\left(L^2h\left[B_0 + \frac{2d_x}{\lambda} \right]+ d_x\right)}\quad\text{with}\quad\iota :=\frac{2L \lambda}{L+\lambda}. 
		\end{flalign*}
	\end{theorem}
    In short, Theorem~\ref{thm:errorbound} gives us the bound~\eqref{eq:errorbigO} and an explicit expression for the proportionality constant as a function of the algorithm's parameters, the model's Lipschitz and concavity constants, the initial condition, and the latent space's dimensionality. It follows that in order to achieve an error of $\mathcal{O}(\epsilon)$, we need only to choose the step size $h$ proportional to $\epsilon^2$, the particle number $N$ proportional to $\epsilon^{-2}$, and the step number $K$ proportional to $\epsilon^{-2}\log(\epsilon^{-1})$. The corresponding computational cost is $\cal{O}(\epsilon^{-4}\log(\epsilon^{-1})[\text{eval. cost of }\ell])$; and we can lower the running time to $\cal{O}(\epsilon^{-2}\log(\epsilon^{-1})[\text{eval. cost of }\ell])$ by parallelizing/distributing computations across the particles (recall that each particle's update is independent of the other particles'). 
	
	Given the definition of the metric $\mathsf{d}$ in \eqref{eq:dmetric}, Theorem~\ref{thm:errorbound} bounds the $\cal{L}^2$ error of the parameter estimates, $\Theta_K^{N,h}$, and, also, the expected $W_2$ distance between the particles' empirical distribution, $Q_K^{N,h}$, and that, $Q_*^N$, of $N$ i.i.d.~samples drawn from the optimal posterior $\pi_{\theta_*}$. We find the latter insightful because, in many ways, $Q_*^N$ is the best one could hope to achieve using an $\cal{O}(N)$-cost Monte Carlo algorithm that returns $N$-sample approximations $\pi_{\theta_*}$.
	The expected $W_2$ distance between $Q^N_*$ itself and the optimal posterior $\pi_{\theta_*}$ is essentially a property of $\pi_{\theta_*}$ and has more to do with the fundamental limitations of approximating $\pi_{\theta_*}$ using $N$-sample ensembles than the method used to obtain the ensemble. To bound this distance (and, by combining this bound with Theorem~\ref{thm:errorbound}, to obtain bounds on the distance between $Q_K^{N,h}$ and $\pi_{\theta_*}$), we point the reader to \citet{Fournier2023}, \citet{Dereich2013}, and references therein. 
	\subsection{The error bound under warm-starts} If we `warm start' PGD (e.g., see \citealp{Dalalyan2017}) by initializing the parameter estimates and the particles at the log-likelihood $\ell$ maxima (i.e., with $\theta_0=0$ and $q_0=\delta_0$), then $B_0$ in~\eqref{eq:pgderrorbound} vanishes and we obtain the following sharper bound. 
	\begin{corollary}[Error bound under warm starts]\label{cor:pgderror} Under the conditions in Theorem~\ref{thm:errorbound}'s premise and with $\iota$ as defined therein, if $(\theta_0,q_0):=(0,\delta_0)$, then, for all $K$ in $\n$,
		\begin{align*} 
			\mathsf{d} ((\Theta_{K}^{N,h},Q_{K}^{N,h}), (\theta_*,Q_*^N)) \leq  \sqrt{h\frac{4h+4/\iota}{\iota}220L^2\left(\frac{2L^2hd_x}{\lambda}+ d_x\right)}+ 2\sqrt{\frac{d_x}{\lambda}}\left(\frac{L}{\lambda \sqrt{N}} + e^{-h\lambda K}\right).
		\end{align*}
	\end{corollary}
	\noindent {\bf Proof}{.
		Combining Theorem~\ref{thm:dflowconvergencehp} together with Proposition \ref{prop:uniformmomentboundflow} in Appendix~\ref{app:sdesolutions}, we find that
		$$\mathsf{d}((\theta_0,q_0),(\theta_*,\pi_{\theta_*}))^2=\lim_{t\rightarrow \infty}\mathsf{d}((\theta_0,q_0),(\theta_t,q_t))^2 \leq \norm{\theta_t}^2+\Ebb{\norm{X_t}^2}\leq \frac{4d_x}{\lambda}.$$ 
	} \hfill\BlackBox \\
    This bound features only PGD's parameters, the model's Lipschitz and concavity constants, and the latent space's dimensionality; and we can be more explicit in our error analysis. To measure the error, we use $\sqrt{\lambda}\mathsf{d}$ rather than $\mathsf{d}$: a common practice in the analysis of Langevin algorithms (e.g., see \citealp[Section 4.1]{Chewi2023}) motivated by $\sqrt{\lambda}\mathsf{d}$'s close relationship with the objectives we wish to minimize~(\ref{eq:exttalagrand},\ref{eq:stdtalagrand},\ref{eq:quadgrowthcondition}). In this metric, and with $\kappa:=L/\lambda$ denoting the model's condition number, Corollary~\ref{cor:pgderror} shows that PGD, with a warm start, achieves an $\cal{O}(\epsilon)$ error as long as the step size satisfies $h\asymp \epsilon^2 (d_x L\kappa)^{-1}$, the number of particles is $\cal{O}(d_x\kappa^2\epsilon^{-2})$, and the number of time steps is $\cal{O}(d_x\kappa^2\log(d_x^{1/2}\epsilon^{-1})\epsilon^{-2})$. The corresponding cost is $\cal{O}(d_x^2\kappa^4\log(d_x^{1/2}\epsilon^{-1})\epsilon^{-4}[\text{eval. cost of }\ell])$ and the running time can be brought down to $\cal{O}(d_x\kappa^2\log(d_x^{1/2}\epsilon^{-1})\epsilon^{-2}[\text{eval. cost of }\ell])$ by parallelizing/distributing over particles. Starting PGD warm is reasonably practical because the cost of running PGD will typically significantly exceed that of maximizing $\ell$ using a first-order method. 
	
	The cost and running time bounds grow with the latent space's dimensionality $d_x$. This is an issue for many models of practical interest where $d_x$ is proportional to the number of datapoints in a large training set. However, as we now explore, it is possible to obtain bounds independent of $d_x$ for a class of such models. 
	\subsection{Dimension-free bounds for models with independent latent variables and conditionally Gaussian observations}
	Often, in machine learning applications, the data $y$ decomposes as a sequence of datapoints $(\tilde{y}_m)_{m=1}^M$, the model features a latent variable  $\tilde{x}_m$ per datapoint $\tilde{y}_m$, the pairs $(\tilde{x}_m,\tilde{y}_m)_{m=1}^M$ are assumed to be i.i.d.~and generated through a mechanism of the sort 
	\begin{equation*}
	\tilde{y}_m=f_\theta(\tilde{x}_m)+\eta_m,\quad \tilde{x}_m\sim p_{\theta}^x(\dif \tilde{x}),\quad \eta_m\sim p^\eta(\dif \eta),\quad\forall m=1,\dots,M;
	\end{equation*}
	where $f_\theta$ denotes a parametrized function mapping from a latent space  $\r^{\Tilde{d_x}}$ to the parameter space, $p_{\theta}^x$ denotes a prior distribution on the latent variables, and the noise has a Gaussian law, $p^\eta$, which is independent of the model parameters $\theta$. Examples include noisy independent component analysis~\citep{Hyvarinen1998}, probabilistic matrix factorization~\citep{Mnih2007}, the generative model behind variational autoencoders~\citep{Kingma2019}, and latent diffusion models~\citep{Vahdat2021,Wehenkel2021,Wang2025}.
	In these cases, the log-likelihood breaks down into a sum of per-datapoint log-likelihoods:
	\begin{equation}\label{eq:minilogliks}\ell(\theta,x)=\sum_{m=1}^M \tilde{\ell}(\theta,\tilde{x}_m; \tilde{y}_m),\end{equation}
	where 
	\begin{equation}\label{eq:minilogliks2}\tilde{\ell}(\theta,\tilde{x}_m; \tilde{y}_m):=\log p^\eta(\tilde{y}_m-f_\theta(\tilde{x}_m))+\log p_{\theta}^x(\tilde{x}_m)\quad\forall m=1,\dots,M.    
	\end{equation}
	Because $p^\eta$ is Gaussian, $(\tilde{x},\theta)\mapsto\tilde{\ell}(\tilde{x},\theta; \tilde{y})$'s Hessian $\nabla^2_{(\theta,\tilde{x})} \tilde{\ell}$ does not depend on $\tilde{y}$. Consequently, if $\tilde{\ell}$ is $\tilde{\lambda}$-strongly concave and its gradient is $\tilde{L}$-Lispchitz,
	\begin{equation}\label{eq:miniregularity}\tilde{\lambda}I\preceq -\nabla^2_{(\theta,\tilde{x})} \tilde{\ell}\preceq\tilde{L}I \end{equation}
	for some positive $\tilde{\lambda}$ and $\tilde{L}$, then $\ell$ is $(M\tilde{\lambda})$-strongly concave and has a $(M\tilde{L})$-Lipschitz gradient. Combining this fact with Corollary~\ref{cor:pgderror}, we obtain the following:
	\begin{corollary}\label{cor:minierrorbounds}
		Let the conditions in Theorem~\ref{thm:errorbound}'s premise hold. Suppose that $(\theta_0,q_0)=(0,\delta_0)$, $\ell$ is as in~\eqref{eq:minilogliks} with $\tilde{l}$ twice-differentiable and satisfying~\eqref{eq:miniregularity}, and $\epsilon>0$. If
		$$h\leq \frac{C_1\epsilon^2}{\tilde{d}_xM},\qquad N\geq \frac{C_2\tilde{d}_x}{\epsilon^{2}},\qquad K \geq \frac{C_3\tilde{d}_x}{\epsilon^{2}}\log\left(\frac{1}{\tilde{d}_x^{1/2}\epsilon}\right),$$
		for some constants $C_1, C_2,C_3>0$ independent of $M$, then
		$$\mathsf{d} ((\Theta_{K}^{N,h},Q_{K}^{N,h}), (\theta_*,Q_*^N))\leq C\epsilon,$$
		for another constant $C>0$ independent of $M$.   
	\end{corollary}
	The corresponding cost is $\cal{O}(\epsilon^{-4}\log(\epsilon^{-1})[\text{eval. cost of }\tilde{\ell}])$, where we are ignoring any dependence on $\tilde{d}_x$ which is usually small for these models.   In particular, the cost depends on $M$ only through the evaluation cost of $\ell$, which cannot be avoided. We can, however, mitigate this by parellizing/distributing over both particles and datapoints, so lowering the running time to $\cal{O}(\epsilon^{-2}\log(\epsilon^{-1})[\text{eval. cost of }\tilde{\ell}])$.

	The corollary also provides some theoretical support for the idea that the PGD algorithm and its derivatives can work well for problems in which $d_x$ is large, even when $N$ is relatively modest, provided that the structure of the model allows the log-likelihood to be decomposed as a sum of functions of low-dimensional contributions arising from individual datapoints. This appeared to be the case in several examples explored in \citet[Section~3]{PGD}, including a `generator network' model for which $d_x\approx 2.5 \times 10^6$ but good performance was obtained with only $10$ particles (similar results were also obtained in \citet[Section 6.3]{Lim2023} using their version of PGD). The model, however, lies beyond the scope of Corollary~\ref{cor:minierrorbounds} because, in that case, $f_\theta$ in~\eqref{eq:minilogliks2} is a non-convex neural network and $\theta$ encompasses its weights. Moreover, \citet{PGD} adapted PGD's step sizes to account for this and also subsampled $\ell$'s gradient to facilitate efficient implementation on a single GPU. Regardless, we view the corollary as a tentative first step in explaining the behaviour observed in \citet{PGD}. 
	\subsection{PGD's Derivation and Theorem~\protect{\ref{thm:errorbound}}'s Proof}\label{sec:errrorboundproof}
	We obtain PGD by approximating the gradient flow~\eqref{eq:pde} via the McKean--Vlasov SDE~\eqref{eq:sde1}. The first step is to approximate the law, $q_t$, of $X_t$ in~\eqref{eq:sde1} with the empirical distribution $Q_t^N:=N^{-1}\sum_{n=1}^N\delta_{X_t^n}$ of $N$ i.i.d.~copies $X_t^1,\dots,X_t^N$ of $X_t$. To obtain these copies, consider
	\begin{align}
		\dif \theta_t = \frac{1}{N}\sum_{n=1}^N\int \nabla_\theta \ell(\theta_t,x)q_t^n(\dif x)\dif t,\quad
		\dif X^n_t=\nabla_x \ell(\theta_t,X^n_t)\dif t+\sqrt{2}\dif W_t^n,\quad\forall n\in[N],  \label{eq:spacediscrsystems1}
	\end{align}
	where $[N]:=\{1,\dots,N\}$, $W^1,\dots,W^N$ denote $N$ independent Brownian motions,  $q_t^n:=\text{Law}(X_t^n)$ for each $n$, and
	we assume that the initial condition $X_0^n$ of each particle $(X_t^n)_{t\geq0}$ is drawn independently from the law $q_0$ of $(X_t)_{t\geq0}$'s initial condition $X_0$. Because the particles all satisfy the same differential equation and have the same initial distribution, they share the same law (in particular,~$q_t^n$ is independent of $n$). Hence, for any given $n$, we can re-write the leftmost equation in~\eqref{eq:spacediscrsystems1} as $\dif \theta_t=\int\nabla_\theta\ell(\theta_t,x)q_t^n(\dif x) \dif t$. Putting this equation together with that for $X_t^n$ in~\eqref{eq:spacediscrsystems1}, we recover~\eqref{eq:sde1} with $X^n_t$ replacing $X_t$ therein. In other words, if $(\theta_t,X^1_t,\dots,X^N_t)_{t\geq0}$ solves~\eqref{eq:spacediscrsystems1}, then $(\theta_t,X_t^n)_{t\geq0}$ solves~\eqref{eq:sde1} for any $n$. Consequently, Proposition~\ref{prop:sdesolutions} tells us that, for any $n$, $(\theta_t,q_t^n)_{t\geq0}$  is a classical solution to the gradient flow \eqref{eq:pde}. Exploiting this link and Theorem~\ref{thm:dflowconvergencehp}, we can bound the distance between $(\theta_t,Q_t^N)$ and $(\theta_*,Q_*^N)$, where  $Q_*^N:=N^{-1}\sum_{n=1}^N\delta_{X_*^n}$ and $X_*^1,\dots,X_*^N$ denotes $N$ i.i.d.~particles with law $\pi_{\theta*}$:
	\begin{lemma}\label{lemma:iidapprox} Under Assumptions \ref{ass:model}(ii--iii) and \ref{ass:strongconcave}--\ref{ass:gradLip}, if $q_0$ belongs to $\cal{P}_2(\r^{d_x})$, then~\eqref{eq:spacediscrsystems1} has a strong solution $(\theta_t,X_t^1,\dots,X_t^N)_{t\geq0}$ and the solution satisfies
		$$\mathsf{d}((\theta_t,Q_t^N),(\theta_*,Q_*^N))\leq \mathsf{d}((\theta_0,q_0),(\theta_*,\pi_{\theta_*}))e^{-\lambda t}\quad\forall t\geq0.$$
	\end{lemma}
	\noindent {\bf Proof}{.
		The argument for the existence of solutions for~\eqref{eq:spacediscrsystems1} is similar to that for~\eqref{eq:sde1} in Proposition~\ref{prop:sdesolutions}'s proof, and we skip it. The error bound follows by optimally coupling the two sets of particles; see Appendix~\ref{app:iidapprox} for details.
	} \hfill\BlackBox \\

	The next step entails approximating $N^{-1}\sum_{n=1}^Nq_t^n(\dif x)$ in~\eqref{eq:spacediscrsystems1} with the particle's empirical distribution $Q_t^N$; which leads us to the following SDE:
	\begin{align}
		\dif \Theta_t^N =&\frac{1}{N}\sum_{n=1}^N\nabla_\theta \ell(\Theta_t^N,\bar{X}_t^{n})\dif t,& \dif \bar{X}_{t}^{n}=&\nabla_x \ell(\Theta_t^N,X_t^{n})\dif t+\sqrt{2}\dif W_t^n,\quad \quad\forall n\in [N].\label{eq:spacediscrsystems2}
	\end{align}
	where $\bar{X}_0^n=X_0^n$ for each $n$ in $[N]$. 
	We can bound the approximation error between (\ref{eq:spacediscrsystems1}) and (\ref{eq:spacediscrsystems2}) as follows:
	\begin{lemma} [Bound on the space discretization error]\label{lemma:propchaos}
		Under Assumptions \ref{ass:strongconcave}--\ref{ass:gradLip}, if $q_0$ belongs to $\cal{P}_2(\r^{d_x})$, then~\eqref{eq:spacediscrsystems2} has a strong solution $(\Theta_t^N,\bar{X}_t^1,\dots,\bar{X}_t^N)_{t\geq0}$ and the solution satisfies
		\begin{equation*}
			\mathsf{d}((\Theta^N_{t},\bar{Q}^N_{t}),(\theta_{t},Q^N_{t}))\leq \frac{L\sqrt{2}}{\lambda \sqrt{N}}\sqrt{\norm{\theta_0}^2 + \Ebb{\norm{X_0}^{2}}+\frac{d_x}{\lambda}} \enskip \forall t\geq0,\,\text{with}\enskip \bar{Q}_t^N:= \frac{1}{N}\sum_{n=1}^N \delta_{\bar{X}_t^n}.
		\end{equation*}
	\end{lemma}
	\noindent {\bf Proof}{.
		\eqref{eq:spacediscrsystems2}  is a standard SDE with a Lipschitz drift and diffusion coefficients, so the existence of its solutions is routine (e.g. \citealt[Theorem 5.2.1]{Oksendal2013}). The error bound follows from a synchronous coupling argument; see Appendix~\ref{app:propchaos} for details.
	} \hfill\BlackBox \\
	
	Lastly, we obtain PGD by discretizing~\eqref{eq:spacediscrsystems2} in time using the  Euler--Maruyama scheme with step size $h$. Given Lemmas~\ref{lemma:iidapprox}~and~\ref{lemma:propchaos}, Theorem~\ref{thm:errorbound} follows from the triangle inequality for $\mathsf{d}$ and the following bound on the discretization error: 
	\begin{lemma} [Bound on the time discretization error]\label{lemma:disc}
		Under Assumptions \ref{ass:strongconcave}--\ref{ass:gradLip}. If $q_0$ belongs to $\cal{P}_2(\r^{d_x})$,  $h\leq 1/(\lambda+L)$, and $A_{0,h}$ is as in Theorem \ref{thm:errorbound}, then:
		\begin{align*}
			\mathsf{d}((\Theta^{N,h}_{K},Q^{N,h}_{K}),(\Theta^N_{Kh},\bar{Q}^N_{Kh})) \leq \sqrt{h}A_{0,h}.
		\end{align*}
	\end{lemma}
	\noindent {\bf Proof}{.
		The bound follows from a synchronous coupling argument; see Appendix \ref{app:discr}.
	} \hfill\BlackBox \\
	
	Other, more refined, discretization procedures such as Randomized Midpoint Discretization \citep{Shen2019,He2020} might result in an algorithm with better dependencies of the error on the dimension $d_x$. 
	\section{Discussion}\label{sec:conclusion}
	In this paper, we theoretically validated PGD (Algorithm~\ref{alg:pgd}) by bounding its error for models with strongly concave log-likelihoods. In order to do this we analyzed the theoretical properties of both PGD  and its continuous-time infinite-particle limit: the gradient flow~\eqref{eq:pde}. To study the latter, we extended certain inequalities well-known in the optimal transport and optimization literatures and several results interlinking them. In doing so, we obtain conditions weaker than strong concavity of the likelihood under which the flow converges at an exponential rate to the set $\cal{M}_*$ of pairs $(\theta_*,\pi_{\theta_*})$ of marginal likelihood maximizers $\theta_*$ and matching posterior distributions $\pi_{\theta_*}$. Under the additional assumptions that the model log-likelihood is strongly concave and its gradient is Lipschitz continuous, we then proved explicit non-asymptotic error bounds for PGD.
	
	Aside from evidencing PGD's well-foundedness, our bounds enable theoretical comparisons between PGD and other maximum marginal likelihood algorithms. For instance, it is interesting to compare the behaviour of the EM algorithm itself and its many variants with PGD.  \citet{Caprio2024} derived non-asymptotic error bounds for the EM and the gradient EM algorithms under the xLSI. Comparing their EM error bound (Corollary 13 therein) to PGD's (Theorem \ref{thm:errorbound}) suggests that the latter is slower, even in the limit of infinite particles. This is consistent with the interpretation of EM as coordinate descent applied to free energy and PGD as a gradient descent: on Euclidean spaces, coordinate descent is typically faster, in terms of number of iterations, in the special cases where it can be implemented~\citep{Beck2013}. \citet{Balakrishnan2017} and \citet{Kunstner2021} also provide non-asymptotic analyses of EM, but a comparison with their results is not immediate.
    
	\citet{Akyildiz2023} analyze the IPLA algorithm which can be viewed as a modification of PGD more amenable to standard analysis. They obtain very similar scalings in  $\epsilon,d_x,d_\theta$ for $h,N,$ and $K$ necessary to achieve an error of $\epsilon$ under warm-start conditions (see Section 3.6 therein) to those obtained here. The main difference between these bounds is that  $N$ must be proportional to $d_x\epsilon^{-2}$ in our case and proportional to $d_\theta\epsilon^{-2}$ in the case of IPLA. Their dependence on $d_\theta$ stems from the extra noise term featuring IPLA's parameter updates; see Algorithm 1 and the proof of Proposition~3 in \citet{Akyildiz2023}. We conjecture that the dependence on $d_x$ present in our results but not those of \citet{Akyildiz2023} simply reflects the fact that our bounds controlling the error of both the parameter estimates and the particle approximation, while theirs only control the former.
	
	Similarly, our bounds could help settle the question of whether it is possible to `accelerate' PGD and achieve faster convergence rates via the use of momentum~\citep{Lim2023} as its the case for both gradient descent~\citep{Nesterov1983} and the unadjusted Langevin algorithm~\citep{Ma2021}. In particular, \citet[Theorem~4.1]{Lim2023} shows that, under the xLSI, the free energy $F$ along the limiting flow of such a `momentum-enriched'  version of PGD converges exponentially fast to $F$'s infimum. Given Theorem~\ref{thm:exttalagrand}, applying our extension of the Talagrand inequality, we can immediately prove convergence in $\mathsf{d}$-distance for flow studied in \citet{Lim2023}. By following arguments similar to those in Section \ref{sec:errrorboundproof}, it might be possible to then also obtain similar results for the momentum-enriched PGD and compare the algorithms' performances.
	
	While our motivation for this work was understanding the theoretical properties of PGD, we find the results in Section~\ref{sec:flowconvergence} interesting in their own right and we believe they might find applications beyond the study of algorithms for maximum marginal likelihood estimation like PGD. For instance, the xLSI is used in \citet{Lim2024} to theoretically motivate a novel algorithm for semi-implicit variational inference. Additionally, for models with $\lambda$-strongly concave log-likelihoods, Theorem~\ref{thm:extbakryemery} gives us a (to the best of our knowledge)  novel upper bound on the  optimal marginal log-likelihood:
\begin{equation}\label{eq:newbounds1}\log(Z_*)\leq \frac{I(\theta,q)}{2\lambda}-F(\theta,q)\quad\forall(\theta,q)\in\cal{M}_2^1.\end{equation}
	Similarly, for models satisfying the xLSI, Theorem~\ref{thm:exttalagrand} yields a bound on the $\mathsf{d}$-distance between $(\theta,q)$ and the optimal set $\cal{M}_*$:
	\begin{equation}\label{eq:newbounds2}
		\lambda\mathsf{d}((\theta,q),\cal{M}_*)\leq \sqrt{I(\theta,q)}\quad\forall(\theta,q)\in\cal{M}_2^1.
	\end{equation}
	It would be interesting to explore the extent to which the right-hand side of~\eqref{eq:newbounds1}, or of \eqref{eq:newbounds2} prove informative for models of practical interest. Moreover, for such models, \eqref{eq:newbounds2} hints at potential (also to the best of our knowledge) new variational inference methods that would minimize $(\theta,\phi)\mapsto I(\theta,q_\phi)$, where  $(q_\phi)_{\phi\in\Phi}$ denotes a parametrized family, rather than the conventional approach of minimizing  $(\theta,\phi)\mapsto F(\theta,q_\phi)$. 
    
        \acks{
        We would like to thank Jen Ning Lim and Paula Cordero Encinar for their helpful comments, and also the anonymous reviewers for their constructive feedback. JK and AMJ acknowledge support from the Engineering and Physical Sciences Research Council (EPSRC; grant EP/T004134/1). AMJ acknowledges further support from the EPSRC (grant EP/R034710/1) and from UK Research and Innovation (UKRI) via grant no. EP/Y014650/1, as part of the ERC Synergy project OCEAN.  RC was funded by the UK Engineering and Physical Sciences Research Council (EPSRC) via studentship 2585619 as part of grant number EP/W523793/1. SP acknowledges support from the Engineering and Physical Sciences Research Council (EPSRC; grant EP/R018561/1).

Data sharing is not applicable to this article as no new data were generated or analyzed.
}
	\appendix

	\section{Proofs for Section~\protect{\ref{sec:flowconvergence}}} \label{app:flowconvergence}
	\subsection{de Bruijn's Identity}\label{app:debruijn}
    \eqref{eq:debruijn} extends de Bruijn's identity 
		(cf. \citet[Proposition 5.2.2]{Bakry2014}). Assumption \ref{ass:solutions_existence} is sufficient for it to hold:
	\begin{lemma} \label{lemma:debruijin}
		If Assumption \ref{ass:solutions_existence} holds, then equation \eqref{eq:debruijn} also holds.
	\end{lemma}
	\noindent {\bf Proof}{.
		The regularity of $(\theta_t,q_t)_{t\geq0}$ allow us to exchange differentiation and integration by the Leibniz rule and compute
		\begin{align*}
			\frac{\dif}{\dif t}\int \log(q_t(x))q_t(\dif x) &=  \int (\log(q_t(x))+1)\nabla_x\cdot\bigg[q_t(x)\nabla_x \log\bigg(\frac{q_t(x)}{\rho_{\theta_t}(x)}\bigg)\bigg]\dif x \\
			&= -\int \iprod{\nabla_x \log(q_t(x))}{\nabla_x \log\bigg(\frac{q_t(x)}{\rho_{\theta_t}(x)}\bigg)} q_t(\dif x)
		\end{align*}
		where we integrated by parts, and
		\begin{align*}
			\frac{\dif}{\dif t}  \int \ell(\theta_t,x)q_t(\dif x) &= \int \iprod{\nabla_\theta \ell(\theta_t,x)}{\dot{\theta_t}}q_t(\dif x) + \int \ell(\theta_t,x)\nabla_x\cdot\bigg[q_t(x)\nabla_x \log\bigg(\frac{q_t(x)}{\rho_{\theta_t}(x)}\bigg)\bigg]\dif x \\
			&= \norm{\int \nabla_\theta \ell(\theta_t,x)q_t(\dif x)}^2 - \int \iprod{\nabla_x \log(\rho_{\theta_t}(x))}{\nabla_x \log\bigg(\frac{q_t(x)}{\rho_{\theta_t}(x)}\bigg)} q_t(\dif x).
		\end{align*}
		Upon re-arranging, this shows $\frac{\dif}{\dif t}F(\theta_t,q_t)=-I(\theta_t,q_t)$.
	} \hfill\BlackBox \\
	
	\subsection{Proof of Theorem~\protect{\ref{thm:exttalagrand}}}\label{app:exttalagrand}
	Here some arguments are similar to those in \citet{Otto2000}. Throughout this appendix, we assume that Assumptions \ref{ass:model}--\ref{ass:solutions_existence}  hold. 
	Let $(\theta,q)$ be any point in $\cal{M}_2$. As we show in Lemma~\ref{lemma:ddissi} below,
	\begin{equation} \label{eq:ddissi}
		\frac{\dif}{\dif t}  \mathsf{d}((\theta_t,q_t),( \theta, q )) \leq \sqrt{I(\theta_t,q_t)}\quad\forall t>0.
	\end{equation}
	By the theorem's premise, the measures $(\rho_\theta(\dif x))_{\theta\in\r^{d_\theta}}$ satisfy the xLSI or, in other words,
	\begin{equation*}
		\sqrt{I( \theta, q )} \leq \frac{I( \theta, q )}{\sqrt{2\lambda[F( \theta, q )-F_*]}}\quad \forall ( \theta, q )\in\Macd.
	\end{equation*}
	By the de Bruijn's identity~\eqref{eq:debruijn}, we find that
	\begin{equation*}
		\frac{I(\theta_t,q_t)}{2\sqrt{{F(\theta_t,q_t)-F_*}}}=-\frac{\dif}{\dif t} \sqrt{F(\theta_t,q_t)-F_*}\quad\forall t>0.
	\end{equation*}
	Combining the above three (in)equalities, we obtain
	\begin{align*}
		\frac{\dif}{\dif t}  \mathsf{d}((\theta_t,q_t),( \theta, q )) &\leq \sqrt{I(\theta_t,q_t)}  \leq  \frac{I(\theta_t,q_t)}{\sqrt{2\lambda[F(\theta_t,q_t)-F_*]}} =
		-\frac{\dif}{\dif t}\sqrt{\frac{2[F(\theta_{t},q_{t})-F_*]}{\lambda}}\quad\forall t>0.
	\end{align*}
	Integrating over $t$ in $(t,t')$ for $t'\geq t\geq0$, yields
	$$
	\mathsf{d}((\theta_{t'},q_{t'}),(\theta,q)) - \mathsf{d}((\theta_t,q_t),(\theta,q)) \leq \sqrt{\frac{2[F(\theta_{t},q_{t})-F_*]}{\lambda}}-\sqrt{\frac{2[F(\theta_{t'},q_{t'})-F_*]}{\lambda}}.
	$$
	If we set $(\theta,q):=(\theta_{t},q_{t})$, then
	\begin{equation} \label{eq:dinequalityforsubseq}\mathsf{d}((\theta_t,q_t),(\theta_{t'},q_{t'})) \leq \sqrt{\frac{2[F(\theta_{t},q_{t})-F_*]}{\lambda}}-\sqrt{\frac{2[F(\theta_{t'},q_{t'})-F_*]}{\lambda}}\quad\forall t'\geq t\geq0.\end{equation}
	As we show in Lemma~\ref{lemma:seqCauchy} below, there exists an increasing sequence $t_1<t_2<\dots$ approaching $\infty$ such that
	\begin{equation}\label{eq:dlim}(\theta_{t_n},q_{t_n})\to(\theta_*,q_*)\quad\text{as}\quad n\to\infty,
	\end{equation}
	for some $(\theta_*,q_*)$ belonging to the optimal set $\cal{M}_*$ in the topology induced by $\mathsf{d}$ on $\cal{M}_2$. Because $F(\theta_{t_n},q_{t_n})$ converges to $F_*$ as $n\to\infty$ (Theorem~\ref{thm:flowconvergence}) and 
	$$(\theta',q')\mapsto\mathsf{d}(( \theta, q ),( \theta', q' ))$$
	is a continuous function on $(\cal{M}_2,\mathsf{d})$ for any given $( \theta, q )$ in $\cal{M}_2$, setting $(t,t'):=(0,t_n)$ in~\eqref{eq:dinequalityforsubseq} and taking the limit $n\to\infty$, then implies that
	$$\mathsf{d}((\theta_0,q_0),\cal{M}_*)\leq  \sqrt{\frac{2[F(\theta_0,q_0)-F_*]}{\lambda}};$$
	since $(\theta_0,q_0)\in\cal{M}_2$ is arbitrary, the claim follows.

	\begin{lemma} \label{lemma:ddissi}\eqref{eq:ddissi} holds.
	\end{lemma}
	\noindent {\bf Proof}{.
		Let $(\theta,q)\in\cal{M}_2$, take $(\theta_t,q_t)_{t\geq 0}$ as per Assumption \ref{ass:solutions_existence}  and define the velocity field $v_t(x):=\nabla_x \log(q_t(x)/\rho_{\theta_t}(x))$. \citet[Theorem 8.4.7]{ambrosio2005} shows that 
		\begin{equation} \label{eq:W2derivative}
			\frac{\dif }{\dif t}\mathsf{d}_{W_2}(q_t,q)^2 = 2\int \iprod{v_t(x_1)}{x_1-x_2}\dif \varrho(x_1,x_2)
		\end{equation}
		where $\varrho$ it an optimal transport plan for $(q_t,q)$.  
		By an application of the Cauchy--Schwarz inequality and the definition of the Wasserstein-2 distance we get
		\begin{align*}
			\frac{\dif }{\dif t}\mathsf{d}_{W_2}(q_t,q)^2  &\leq 2\sqrt{\int \norm{x_1-x_2}^2 \dif \varrho(x_1,x_2)\int \norm{v_t(x)}^2 q_t(\dif x)} \\
            &= 2\mathsf{d}_{W_2}(q_t,q)\sqrt{\int \norm{v_t(x)}^2 q_t(\dif x)}
		\end{align*}
        (this inequality may also be proved directly using the Benamou--Brenier's formula).
        On the other hand, $\frac{\dif}{\dif t}\norm{\theta_t-\theta}^2 = 2\iprod{\dot{\theta_t}}{\theta_t-\theta} \leq 2\lVert\dot{\theta_t}\rVert\norm{\theta_t-\theta}$, and by these estimates, and then again by Cauchy--Schwarz, 
		\begin{align*}
			\frac{\dif }{\dif t}\mathsf{d}((\theta_t,q_t),( \theta, q )) = \frac{\frac{\dif}{\dif t}\mathsf{d}((\theta_t,q_t),( \theta, q ))^2}{2\mathsf{d}((\theta_t,q_t),( \theta, q ))} \leq \sqrt{\lVert\dot{\theta_t}\rVert^2 + \int \norm{v_t(x)}^2 q_t(\dif x)} = \sqrt{I(\theta_t,q_t)}.
		\end{align*}
	} \hfill\BlackBox \\
	
	\begin{lemma} \label{lemma:seqCauchy}For any increasing sequence $t_1<t_2<\dots$ approaching $\infty$, \eqref{eq:dlim} holds for some $(\theta_*,q_*)$ in
		$\cal{M}_*$.
	\end{lemma}
	\noindent {\bf Proof}{.
		Inequality~\eqref{eq:dinequalityforsubseq} tells us that the sequence is Cauchy in $(\Macd,\mathsf{d})$ and, consequently, in  $(\cal{M}_2:=\r^{d_\theta}\times\cal{P}_2(\r^{d_x}),\mathsf{d})$. Because both $(\r^{d_\theta},\mathsf{d}_E)$ and $(\cal{P}_2(\r^{d_x}),\mathsf{d}_{W_2})$ are complete metric spaces \citep[Theorem 6.18]{Villani2009}, so is $(\cal{M}_2,\mathsf{d})$. Hence, $(\theta_{t_n},q_{t_n})_{n\in\n}$ converges to a limit $(\theta_\infty,q_\infty)$ in $\mathcal{M}_2$. As we show in Lemma~\ref{lemma:Flsc} below, the free energy $F$ is lower semicontinuous on $(\cal{M}_2,\mathsf{d})$. Consequently, Theorem~\ref{thm:flowconvergence} implies that
		$$F(\theta_\infty,q_\infty)\leq \liminf_{n\to\infty} F(\theta_{t_n},q_{t_n})=F_*,$$
		where the equality holds because $(\theta_{t_n},q_{t_n})_{n\in\n}\subseteq\Macd$. Since it also holds $F(\theta_\infty,q_\infty)\geq F_*$, the claim follows.
	} \hfill\BlackBox \\

	\begin{lemma} \label{lemma:Flsc} $F$ is lower semicontinuous on $(\cal{M}_2,\mathsf{d})$.
	\end{lemma}
	\noindent {\bf Proof}{.
		Given any $\pi$ in $\cal{P}_2(\r^{d_x})$, the duality formula for the $\mathrm{KL}$ divergence (e.g., \citealp[Lemma 9.4.4]{ambrosio2005}) reads
		$$\mathrm{KL}(q||\pi)=\sup \left\{\int\phi(x)q(\dif x) - \int (e^{\phi(x)}-1)\pi(\dif x):\phi\in \cal{C}_b(\r^{d_x},\r)\right\},$$
		where $\cal{C}_b(\r^{d_x},\r)$ denotes the set of real valued continuous bounded functions with domain $\r^{d_x}$.
		Consequently,
		\begin{align*}
			F( \theta, q )=\mathrm{KL}(q||\pi_\theta) -\log(Z_\theta)=\sup \{G(\theta,q,\phi) :\phi\in \cal{C}_b(\r^{d_x},\r)\}-\log(Z_\theta),
		\end{align*}
		where $G(\theta,q,\phi):=\int\phi(x)q(\dif x) - \int (e^{\phi(x)}-1)\pi_\theta(\dif x)$.  Because Assumption~\ref{ass:model}(i) implies that $\theta\mapsto\log(Z_\theta)$ is continuous and the pointwise supremum of a family of lower semicontinuous (l.s.c.) functions is l.s.c., we need only show that $( \theta, q )\mapsto G(\theta,q,\phi)$ is l.s.c. for each $\phi$ in  $\cal{C}_b(\r^{d_x},\r)$. 
		Since  the Wasserstein-2 topology is finer than the weak topology, it suffices to show that $( \theta, q )\mapsto G(\theta,q,\phi)$ is continuous in the weak topology for any given $\phi$ in  $\cal{C}_b(\r^{d_x},\r)$. Consider a sequence $(\theta_n,q_n)_{n\in\n}$ such that $\theta_n\rightarrow \theta$ and $q_n\rightarrow q$ weakly. By definition of weak convergence, $\int \phi(x)q_n(\dif x)\rightarrow \int\phi(x)q(\dif x)$. Furthermore,  since $x\mapsto(e^{\phi(x)}-1)\in\cal{C}_b(\r^{d_x},\r)$, we have $|\int (e^{\phi(x)}-1)\pi_{\theta_n}(\dif x) - \int (e^{\phi(x)}-1)\pi_{\theta}(\dif x) | \leq c_\phi\int|\pi_{\theta_n}(x) - \pi_{\theta}(x)|\dif x$ for some $c_\phi<\infty$. Since $\theta \mapsto \pi_\theta$ is continuous by Assumption \ref{ass:model}(i), an application of Scheff{\'e}'s lemma concludes the proof.
	} \hfill\BlackBox \\
	\subsection{Proof of Theorem~\protect{\ref{thm:extbakryemery}}}\label{app:extbakryemery}
	We preface the proof of Theorem \ref{thm:extbakryemery} with three auxiliary results which describe variations of the free energy along geodesics in $(\cal{M}_2,\mathsf{d})$. 
	\begin{definition} \label{def:geodesics}
		A curve $\gamma(t):t\in[0,1]\mapsto \mathcal{M}_2$ is a (constant speed) geodesic if
		\begin{equation*}
			\mathsf{d}(\gamma(s),\gamma(t))=(t-s)\mathsf{d}(\gamma(0),\gamma(1)), \quad \forall 0\leq s\leq t\leq 1.
		\end{equation*}
	\end{definition}
	For a probability measure $\mu\in\mathcal{P}(\r^{d_x})$ and a measurable map $T$ with domain $\r^{d_x}$, $T_{\#}\mu = \mu \circ T^{-1}$ indicates the pushforward of $\mu$ by $T$. 
	\begin{lemma} \label{lemma:geodesicsinM}
		A curve $\gamma(t):t\in[0,1]\mapsto \mathcal{M}_2$ is a geodesic if and only if $\gamma(t)=(\gamma_{\theta}(t),\gamma_{q}(t))$, where $\gamma_{\theta}$ is a geodesic in $(\r^{d_\theta},\norm{\cdot})$ and $\gamma_{q}$ is a geodesic in $(\cal{P}_2(\r^{d_x}),\mathsf{d}_{W_2})$. In particular, if $\gamma(t)$ is a geodesic in $(\mathcal{M}_2,\mathsf{d})$ connecting $( \theta, q )$ and $(\theta',q')$ then $\gamma(t)=(\gamma_{\theta}(t)=(1-t)\theta+t\theta',\gamma_{q}(t)=(h_t)_{\#}\varrho)$ where $\varrho$ is a $Wasserstein-2$ optimal transport plan for $q,q'$ and $h_t(x,x')=(1-t)x+tx'$. Furthermore, if $q$ has a density w.r.t. the Lebesgue measure, we can also write $\gamma_{q}(t)=((1-t)\textrm{id}+t\nabla_x \Phi)_\# q$ for some convex function $\Phi$.
	\end{lemma}
	\noindent {\bf Proof}{. 
		The first claim is almost immediate from the definition, also see \citet[Lemma 3.6.4]{Burago2001}. The second claim follows from the first using the characterization of geodesics in $(\cal{P}_2(\r^{d_x}),\mathsf{d}_{W_2})$---see e.g. \citet[Theorem 5.27]{santambrogio2015}. The last claim is Brenier's Theorem \citep[Theorem 9.4]{Villani2009}.
	} \hfill\BlackBox \\
	
	The following lemma uses similar arguments to \citet[Lemma 9]{Cheng2018}, and it connects the strong concavity of the log-likelihood function with strong geodesic convexity of $F$.
	\begin{lemma} \label{lemma:Fgeodesicconv}
		If Assumption \ref{ass:strongconcave} holds, $F$ is $\lambda$-strongly geodesically convex i.e. for every $( \theta, q ),(\theta',q')$ there is constant speed geodesic $\gamma:[0,1]\to\cal{M}_2$ connecting $\gamma(0)=( \theta, q )$ and $\gamma(1)=(\theta',q')$ such that $t\mapsto F(\gamma(t))$ is $\lambda$-strongly convex:
		\begin{equation} \label{eq:Fgeodesicconv}
			F(\gamma(t))\leq (1-t)F( \theta, q )+tF(\theta',q')-\frac{\lambda t(1-t)}{2}\mathsf{d}(( \theta, q ),(\theta',q'))^2.
		\end{equation}
	\end{lemma}
	\noindent {\bf Proof}{.
		Let $\gamma=(\gamma_\theta,\gamma_q)$ be a geodesic in $\mathcal{M}_2$ connecting $( \theta, q )$ and $(\theta',q')$, and recall that by Lemma \ref{lemma:geodesicsinM} $\gamma_\theta$ and $\gamma_q$ are geodesics in $(\r^{d_\theta},\norm{\cdot})$ and $(\cal{P}_2(\r^{d_x}),\mathsf{d}_{W_2})$, respectively. Since $F( \theta, q )=\int \log(q(x))q(\dif x)-\int \ell(\theta,x)q(\dif x)$, and since $q\mapsto \int \log(q(x))q(\dif x)$ is convex along $\gamma_q$ by \citet[Theorem 7.28]{santambrogio2015}, we just need to show that $V( \theta, q ):=-\int \ell(\theta,x)q(\dif x)$ is $\lambda$-strongly convex along $\gamma$. Using the representation in Lemma \ref{lemma:geodesicsinM} above and then applying Assumption~\ref{ass:strongconcave} we obtain
		\begin{align*}
			V(\gamma(t))=&-\int \ell(\gamma_\theta(t),x)\gamma_q(t)(x)\dif x\\ =& -\int \ell((1-t)\theta+t\theta',(1-t)x+tx')\varrho(x,x')\dif x\dif x' \\
			\leq& -\int\int ((1-t)\ell(\theta,x)+t\ell(\theta',x'))\varrho(x,x')\dif x\dif x' \\
			&- \frac{\lambda t(1-t)}{2} \int \left(\norm{\theta-\theta'}^2+\norm{x-x'}^2\right)\varrho(x,x')\dif x\dif x' \\
			=& (1-t)V( \theta, q )+tV(\theta',q') - \frac{\lambda t(1-t)}{2}\mathsf{d}(( \theta, q ),(\theta',q'))^2.
		\end{align*}
	} \hfill\BlackBox \\
	
	The next result supplies an estimate for the right lower derivative of the free energy along geodesics in $\mathcal{M}_2$ (to compare with \citealp[Theorem 20.1]{Villani2009}, or \citealp[Equation 51]{Otto2000}).
	\begin{lemma} \label{lemma:Fvariations}
		Let $\gamma$ be a geodesic connecting $( \theta, q ),(\theta',q')\in\Macd$. If $\ell$ is differentiable, then  
		\begin{equation*} \label{eq:Fvariations}
			\liminf_{t\rightarrow 0^+} \frac{F(\gamma(t))-F(\gamma(0))}{t} \geq \iprod{\nabla_x \Phi-id}{\nabla_x \delta_q F(\theta,q)}_q+\iprod{\theta'-\theta}{\nabla_\theta F(\theta,q)}
		\end{equation*}
		where $\delta_q F(\theta,q)=\log(q(x))-\ell(\theta,x)$ is $F$'s first variation and $\nabla_x \Phi$ is the optimal transport map from $q$ to $q'$.
	\end{lemma}
	\noindent {\bf Proof}{.
		Let us consider the geodesic $\gamma(t)$ given by $\gamma(t):=(\theta_t,q_t)$ with $q_t(\dif x):=((1-t)x+t\nabla_x\Phi(x))\#q(\dif x)$ and $\theta_t=(1-t)\theta + t\theta'$ --- see Lemma \ref{lemma:geodesicsinM}. Since $\Phi$ is a convex function (Lemma \ref{lemma:geodesicsinM}), by Alexandrov's theorem \citep[Theorem 14.24]{Villani2009} $\Phi$ has a second derivative almost everywhere. Let $\nabla^2_x$ denote the Hessian with respect to the variable $x$. We use the change of variables $x\rightarrow (1-t)x+t\nabla_x\Phi(x)$ and then the facts $\theta_t=(1-t)\theta + t\theta'$ and $q_t(\dif x):=((1-t)x+t\nabla_x\Phi(x))\#q(\dif x)\Rightarrow q_t((1-t)x+t\nabla_x\Phi(x))=q(x)/\det((1-t)I_{d_x} + t\nabla^2_x\Phi(x))$, so 
		\begin{align*} 
			&F(\gamma(t))=F(\theta_t,q_t)=\int q_t(x)(\log(q_t(x)) - \ell(\theta_t,x))\dif x \nonumber \\
			&= \int q_t((1-t)x+t\nabla_x\Phi(x))\log(q_t((1-t)x+t\nabla_x\Phi(x)) - q_t((1-t)x+t\nabla_x\Phi(x))\\ 
			&\cdot \ell(\theta_t,(1-t)x+t\nabla_x\Phi(x))  \det((1-t)I_{d_x} + t\nabla^2_x\Phi(x))\dif x \nonumber \\
			&=\int q(x) \log\left(\frac{q(x)}{\det((1-t)I_{d_x} + t\nabla^2_x\Phi(x))}\right)\dif x-q(x)\ell(\theta_t,(1-t)x+t\nabla_x\Phi(x))\dif x \\
			&=: \int q(x)\log(q(x))\dif x + F_2(\gamma(t)) + F_3(\gamma(t)). \nonumber
		\end{align*}
		Now, by Fatou's lemma,
		\begin{equation*}
			\liminf_{t\rightarrow 0^+} \frac{F_2(\theta_t,q_t)-F_2(\theta_0,q_0)}{t} \geq \int \liminf_{t\rightarrow 0^+}\frac{\log(\det(I_{d_x}))-\log(\det((1-t)I_{d_x}+t\nabla^2_x\Phi(x)))}{t}q(\dif x)
		\end{equation*}
		and by positive semi-definiteness of $(1-t)I_{d_x}+t\nabla^2_x\Phi(x)$, the arithmetic mean--geometric mean inequality yields that
		\begin{align*}
			\det((1-t)I_{d_x}+t\nabla^2_x\Phi(x)) \leq \bigg(\frac{\textrm{Tr}((1-t)I_{d_x}+t\nabla^2_x\Phi(x))}{d_x}\bigg)^{d_x} = \bigg((1-t)+\frac{t}{d_x}\Delta_x\Phi(x)\bigg)^{d_x}
		\end{align*}
		where $\Delta_x$ denotes the Laplacian. It follows that 
		\begin{align*}
			\liminf_{t\rightarrow 0^+} \frac{F_2(\theta_t,q_t)-F_2(\theta_0,q_0)}{t} &\geq \int \liminf_{t\rightarrow 0^+}\bigg[-\frac{d_x}{t}\log\bigg((1-t)+\frac{t}{d_x}\Delta_x\Phi(x)\bigg) \bigg]q(\dif x) \\
			&= -\int (\Delta_x\Phi(x)-d_x)q(\dif x).
		\end{align*}
		Next, we estimate $F_3(\gamma(t))$'s lower derivative as 
		\begin{align*}
			&\liminf_{t\rightarrow 0^+} \frac{F_3(\theta_t,q_t)-F_3(\theta_0,q_0)}{t} \geq \int  \liminf_{t\rightarrow 0^+}\bigg[\frac{\ell(\theta,x)-\ell(\theta_t,(1-t)x+t\nabla_x\Phi(x))}{t}\bigg]q(\dif x) \\
			&= -\int \iprod{\nabla\ell(\theta,x)}{(\theta'-\theta,\nabla_x\Phi(x)-x)}q(\dif x)
		\end{align*}
		where the limit exchange is again justified by Fatou's lemma. Putting these results together, integrating by parts, and using the divergence theorem then gives
		\begin{align*}
			&\liminf_{t\rightarrow 0^+} \frac{F(\theta_t,q_t)-F(\theta_0,q_0)}{t} \\
			&\geq -\int q(x)(\Delta_x\Phi(x)-d_x)\dif x - \int q(x)\iprod{\nabla\ell(\theta,x)}{(\theta'-\theta,\nabla_x\Phi(x)-x)}\dif x \\
			&= \int \iprod{\nabla_x q(x)}{\nabla_x \Phi(x)-x} \dif x - \int q(x)\iprod{\nabla\ell(\theta,x)}{(\theta'-\theta,\nabla_x\Phi(x)-x)}\dif x \\ 
			&= \int \iprod{\nabla_x q(x)}{\nabla_x \Phi(x)-x} - q(x)\iprod{\nabla_x \ell(\theta,x)}{\nabla_x\Phi(x)-x}\dif x \\
            &-\int \iprod{\theta'-\theta}{\nabla_\theta \ell(\theta,x)}q(x)\dif x\\ 
			&= \int \iprod{\nabla_x \log\left(\frac{q(x)}{\rho_\theta(x)}\right)}{\nabla_x \Phi(x)-x}q(x)\dif x -\int \iprod{\theta'-\theta}{\nabla_\theta \ell(\theta,x)}q(x)\dif x 
		\end{align*}
	} \hfill\BlackBox \\
	
	\noindent {\bf Proof of Theorem \ref{thm:extbakryemery}.}{
		Take any two points $( \theta, q )$ and $(\theta',q')$ in $\Macd$ and consider a geodesic $\gamma=(\gamma_\theta,\gamma_q)$ connecting those. Since under Assumption \ref{ass:model}, $\ell$ is differentiable, taking the right lower derivative w.r.t. $t$ at 0 in \eqref{eq:Fgeodesicconv} and using Lemma \ref{lemma:Fvariations} we obtain
		\begin{align} \label{eq:fnuaifnauifa}
			&\iprod{\nabla_x \Phi-id}{\nabla_x \delta_q F(\theta,q)}_q+\iprod{\theta'-\theta}{\nabla_\theta F(\theta,q)} \nonumber \\
            &\leq \liminf_{t\rightarrow 0^+}\frac{F(\gamma(t))-F(\gamma(0))}{t} \leq F(\theta',q')-F( \theta, q )-\frac{\lambda}{2}\mathsf{d}(( \theta, q ),(\theta',q'))^2.
		\end{align}
		Because $\norm{\nabla_x \Phi-id}_q^2 = \mathsf{d}_{W_2}(q,q')^2$, setting $(\theta',q'):=(\theta_*,\pi_{\theta_*})$ and using the Cauchy--Schwartz inequality yields
		\begin{align*}
			F( \theta, q )-F_*\leq& - \iprod{\nabla_x \Phi-id}{\nabla_x \delta_q F(\theta,q)}_q-\iprod{\theta'-\theta}{\nabla_\theta F(\theta,q)} 
            - \frac{\lambda}{2}\mathsf{d}(( \theta, q ),(\theta',q'))^2 \\
            \leq &+ \mathsf{d}(( \theta, q ),(\theta',q'))\sqrt{\norm{\nabla_x \delta_q F(\theta,q)}^2_q + \norm{\nabla_\theta F(\theta,q)}^2 } - \frac{\lambda}{2}\mathsf{d}(( \theta, q ),(\theta',q'))^2\\
            = &+\mathsf{d}(( \theta, q ),(\theta',q'))\sqrt{I(\theta,q)} - \frac{\lambda}{2}\mathsf{d}(( \theta, q ),(\theta',q'))^2.
		\end{align*}
		Now we use Young's inequality $ab\leq a^2/2\epsilon+\epsilon b^2/2$ with $\epsilon=\lambda>0$ on $\mathsf{d}((\theta,q),(\theta',q'))\sqrt{I(\theta,q)}$ and the claim follows.
	} \hfill\BlackBox \\
	
	\subsection{Proof of Theorem \ref{thm:dflowconvergencehp}} \label{app:dflowconvproof}
	Let $t>0$. Consider a geodesic $\gamma=(\gamma_\theta,\gamma_q)$ connecting the points $(\theta_t, q_t)$ and $(\theta_*,\pi_{\theta_*})$ in $\Macd$. 
	Setting $(\theta,q)=(\theta_t,q_t)$ and $(\theta',q')=(\theta_*,\pi_{\theta_*})$ in \eqref{eq:fnuaifnauifa}, and since under Assumptions \ref{ass:model}(iii) and \ref{ass:strongconcave} the xLSI holds (Theorem \ref{thm:extbakryemery}), we can use the extended Talagrand-type inequality (Theorem \ref{thm:exttalagrand}) to obtain
	\begin{align*}
		&\iprod{\nabla_x \Phi-id}{\nabla_x \delta_q F(\theta_t,q_t)}_{q_t}+\iprod{\theta_\star-\theta_t}{\nabla_\theta F(\theta_t,q_t)} \\
        &\leq F_* - F(\theta_t,q_t) - \frac{\lambda}{2}\mathsf{d}((\theta_t, q_t),(\theta_*,\pi_{\theta_*}))^2 \leq -\lambda\mathsf{d}((\theta_t, q_t),(\theta_*,\pi_{\theta_*}))^2.
	\end{align*}
	Let $\varrho$ be the optimal transport plan for $(q_t,\pi_{\theta_*})$ and let $v_t(x)=\nabla_x \log(q_t(x)/\rho_{\theta_t}(x))$. Recall that $\varrho=(id\times\nabla_x\Phi)_{\#}q_t$ by Brenier's Theorem. Combining the above inequality with \eqref{eq:W2derivative} we write
	\begin{align*}
		&\frac{\dif }{\dif t}\mathsf{d}((\theta_t, q_t),(\theta_*,\pi_{\theta_*}))^2 = 2\int \iprod{v_t(x_1)}{x_1-x_2}\dif \varrho(x_1,x_2) + 2\iprod{\dot{\theta_t}}{\theta_t-\theta_*} \\
		&=2 \iprod{\nabla_x \Phi-id}{\nabla_x \delta_q F(\theta_t,q_t)}_{q_t}+2\iprod{\theta_\star-\theta_t}{\nabla_\theta F(\theta_t,q_t)}
		\leq -2\lambda\mathsf{d}((\theta_t, q_t),(\theta_*,\pi_{\theta_*}))^2,
	\end{align*}
	upon which the result follows via Gr\"onwall's inequality. \hfill\BlackBox \\

	\section{Proofs for Section~\protect{\ref{sec:pgdconvergence}}} \label{app:pgdconvergence}
	\subsection{Proof of Proposition~\protect{\ref{prop:sdesolutions}}}\label{app:sdesolutions}
	We break down the proof of Proposition \ref{prop:sdesolutions} into three steps: (i) in Lemma \ref{lemma:sdesol} we prove that under our assumptions the SDE \eqref{eq:sde1} has an unique strong solution; (ii) in Lemma \ref{lemma:solregularity} we prove some regularity properties for the law of the solution; and finally (iii) in Lemma \ref{lemma:sdepdesol} we prove that such law provides a classical solution to \eqref{eq:pde} satisfying Assumption \ref{ass:solutions_existence}.  
	In this section, Assumptions \ref{ass:model}(ii) and \ref{ass:gradLip} are taken to hold.
	\begin{lemma} \label{lemma:sdesol}
		For all $T>0$, if $(\theta,X)$ belongs to $\r^{d_\theta}\times \cal{L}^2(\r^{d_x})$, the SDE \eqref{eq:sde1} has an unique strong solution $(\theta_t,X_t)_{t\leq T}$ on $[0,T]$ on such that $(\theta_0,X_0)=(\theta,X)$.
	\end{lemma}
\noindent {\bf Proof}{.
This is an adaptation of the arguments in \citet[Theorem 1.7]{Carmona2016}, \citet[Proposition 1]{Chaintron2022}, and \citet[Proposition 3.1]{Lim2023}. 
		Fix any $T>0$, $\theta_0$ in $\r^{d_\theta}$, and $X_0$ in $\cal{L}^2(\r^{d_x})$.   Let $\cal{C}([0,T],\cal{S})$ be the space of continuous functions from $[0,T]$ to a metric space $\cal{S}$. It is straightforward to check that our Lipschitz assumption on $\ell$'s gradient, implies that $\ell(\theta,x)$ is $\nu_t$-integrable for any $t$,  $\theta$, and $\nu$ in $\cal{C}([0,T],\cal{P}_2(\r^{d_x}))$. Consequently, for any given such $\nu$, 
		\begin{equation}\label{eq:simplesde}
			\dif (\theta_t^\nu,X_t^\nu) = b^\nu(\theta_t^\nu,X_t^\nu,t)\dif t + \sigma \dif W_t, \quad (\theta_0^\nu,X_0^\nu)=(\theta_0,X_0),
		\end{equation}
with $b^\nu(\theta,x,t):=[\int \nabla_\theta \ell(\theta,z')\nu_t(\dif z),\nabla_x \ell(\theta,x)]^\intercal$ and $\sigma := \sqrt{2}\textup{diag}(0_{d_\theta},1_{d_x})$, is well-defined. Assumption~\ref{ass:gradLip} and Jensen's inequality imply that $b^\nu$ is Lipschitz in $(\theta,x)$, uniformly over $t$:
		\begin{align*}
			\norm{b^\nu(\theta,x,t)-b^\nu(\theta',x',t)}^2 &\leq\int\norm{\nabla_\theta\ell(\theta,z)-\nabla_\theta\ell(\theta',z)}\nu_t(\dif z)^2+ \norm{\nabla_x\ell(\theta,x)-\nabla_x\ell(\theta',x')} ^2\\
  & \leq L^2\norm{\theta-\theta'}^2+L^2\norm{(\theta,x)-(\theta',x')}^2
  \\&\leq 2L^2\norm{(\theta,x)-(\theta',x')}^2, \quad (\theta,x),(\theta',x')\in \r^{d_\theta}\times \r^{d_x}, t\leq T.
		\end{align*}
Moreover, for any point $(\theta',x')$  in $\r^{d_\theta}\times\r^{d_x}$,
		\begin{align*}
		 &\norm{b^\nu(\theta,x,t)}^2 \\
         &\leq \int\norm{\nabla_\theta\ell(\theta,x)-\nabla_\theta\ell(\theta',x')}^2\nu_t(\dif x)+ \norm{\nabla_x\ell(\theta,x)-\nabla_x\ell(\theta',x')}^2 + \norm{\nabla \ell(\theta',x')}^2 \\
		 &\leq 2L\norm{(\theta,x)-(\theta',x')}^2 + a_T, \quad  (\theta,x)\in \r^{d_\theta}\times \r^{d_x}, t\leq T,
		\end{align*}
where $a_T:=L\sup_{t\leq T}\norm{\theta-\theta'}^2+\int \norm{x-x'}^2\nu_t(\dif x)+\norm{\nabla \ell(\theta',x')}^2<\infty$. 
Consequently, \citet[Theorem 5.2.1.]{Oksendal2013} tells us that \eqref{eq:simplesde} has a unique strong solution $(\theta_t^\nu,X_t^\nu)_{t\leq T}$ over $[0,T]$. Furthermore, $(\textup{Law}(X_t^\nu))_{t\leq T}$ belongs to $\cal{C}([0,T],\cal{P}_2(\r^{d_x}))$: combining the above with Jensen's inequality,
    \begin{align*}
        \Ebb{\norm{\theta_t^\nu}^2+\norm{X_t^\nu}^2} &= \Ebb{\norm{\int_0^t b^\nu(\theta_s^\nu,X_s^\nu,s)\dif s}^2} \\ 
        &\leq 2T\bigg(L\int_0^t\Ebb{\norm{\theta_s^\nu-\theta'}^2 + \norm{X_s^\nu-x'}^2} \dif s + a_T\bigg) \\
        &\leq 2T \bigg(2L\int_0^T \Ebb{\norm{\theta_s^\nu}^2+\norm{X_s^\nu}^2} \dif s + a_T+TL(\norm{\theta'}^2+\norm{x'}^2) \bigg), \hspace{0.4em} \forall t\leq T;
    \end{align*}
and it follows from  Gr\"onwall's inequality that $\Ebb{\norm{\theta_t^\nu}^2+\norm{X_t^\nu}^2}\leq 2T(a_T+TL(\norm{\theta'}^2+\norm{x'}^2)) + 4Le^TT<\infty$ for all $t\leq T$. Now consider the function  
		\begin{equation*}
			\Psi_T:\cal{C}([0,T],\cal{P}_2(\r^{d_x})) \mapsto \cal{C}([0,T],\cal{P}_2(\r^{d_x})), 
		\end{equation*}
mapping $\nu$ to $(\textup{Law}(X_t^\nu))_{t\leq T}$. Note that, if $(\theta_t,X_t)_{t\leq T}$ is a strong solution of \eqref{eq:sde1}, then $(\textrm{Law}(X_t))_{t\leq T}$ is a fixed point of $\Psi_T$ by the latter's definition. Consequently, \eqref{eq:sde1} has no more solutions than $\Psi_T$ has fixed points. On the other hand, if $(q_t)_{t\leq T}$ is a fixed point of $\Psi_T$, then there is some $(\theta_t,X_t)_{t\leq T}$ solving \eqref{eq:simplesde} with $\nu$ equal to $q$ such that $(\textrm{Law}(X_t))_{t\leq T}$ also equals $q$. In other words, there is some $(\theta_t,X_t)_{t\leq T}$ that solves~\eqref{eq:sde1}. In short, to prove the existence and uniqueness of~\eqref{eq:sde1}'s solutions, we need only argue that $\Psi_T$ has a unique fixed point. To this end, consider the following metric on $\cal{C}([0,T],\cal{P}_2(\r^{d_x}))$:
        \begin{equation*}
			\mathsf{d}_{W_2}^T(\nu,\nu'):=\sup_{t\leq T} \mathsf{d}_{W_2}(\nu_t,\nu_t').
		\end{equation*}
If we show that, for some integer $k$, the $k$-fold composition $\Psi_T^k$ of $\Psi_T$ with itself is a contraction w.r.t.~$\mathsf{d}_{W_2}^T$, the Banach--Caccioppoli fixed-point theorem~\citep[Theorem 5.1]{Kreyszig1978} will imply that $\Psi_T$ has a unique fixed point in  $\cal{C}([0,T],\cal{P}_2(\r^{d_x}))$. Fix any two elements $\nu$ and $\nu'$ of $\cal{C}([0,T],\cal{P}_2(\r^{d_x}))$, denote with $(\theta_t^\nu,X_t^\nu)_{t\leq T}$ and $(\theta_t^{\nu'},X_t^{\nu'})_{t\leq T}$ the corresponding SDE \eqref{eq:simplesde} solutions, and let $c$ denotes a rolling constant independent of those, which value might change from line to line. Applying in order Jensen's inequality, $\nabla\ell$'s Lipschitz continuity, and the Kantorovich--Rubinstein duality formula, 		
    \begin{align*}
			\Ebb{\norm{\theta_t^{\nu}-\theta_t^{\nu'}}^2+\norm{X_t^{\nu}-X_t^{\nu'}}^2} &\leq \Ebb{\norm{\int_0^t(b^{\nu}(\theta_s^{\nu},X_s^{\nu},s)-b^{\nu'}(\theta_s^{\nu'},X_s^{\nu'},s))\dif s}^2} \\
			&\leq T\int_0^t\Ebb{\norm{b^{\nu}(\theta_s^{\nu},X_s^{\nu},s)-b^{\nu'}(\theta_s^{\nu'},X_s^{\nu'},s)}^2\dif s} \\
			&\leq cT\int_0^t \Ebb{\norm{\theta_s^{\nu}-\theta_s^{\nu'}}^2 + \norm{X_s^{\nu}-X_s^{\nu'}}^2 + \mathsf{d}_{W_1}(\nu_s,\nu_s')^2}\dif s
		\end{align*}
  Because $\mathsf{d}_{W_1}\leq \mathsf{d}_{W_2}$ \citep[Remark 6.6]{Villani2009}, 
		\begin{align*}
			\int_0^t \mathsf{d}_{W_1}(\nu_s,\nu_s')^2\dif s \leq \int_0^t \sup_{r\leq s} \mathsf{d}_{W_1}(\nu_r,\nu_r')^2\dif s \leq \int_0^T \mathsf{d}_s(\nu_{[0,s]},\nu'_{[0,s]})^2\dif s,
		\end{align*}
  where $\nu_{[0,s]},\nu_{[0,s]}'$ respectively denote $\nu,\nu'$'s restriction to $[0,s]$. Combining the above two and applying Gr\"onwall's inequality, we find that
		\begin{align*}
			\Ebb{\norm{\theta_t^{\nu}-\theta_t^{\nu'}}^2+\norm{X_t^{\nu}-X_t^{\nu'}}^2} 
			\leq c_T\int_0^T \mathsf{d}_s(\nu_{[0,s]},\nu_{[0,s]}')^2\dif s
		\end{align*}
		with $c_T:=cTe^T $. Taking supremums over $t\leq T$, we find that
        \begin{align*}
           \mathsf{d}_{W_2}^T(\Psi_T(\nu),\Psi_T(\nu'))^2 &=\sup_{t\leq T} \mathsf{d}_{W_2}(\nu_t,\nu_t')\leq \sup_{t\leq T}\Ebb{\norm{X_t^\nu-X_t^{\nu'}}^2}\\
           &\leq \sup_{t\leq T}\left(\Ebb{\norm{\theta_t^{\nu}-\theta_t^{\nu'}}^2+\norm{X_t^{\nu}-X_t^{\nu'}}^2} \right)\leq c_T\int_0^T \mathsf{d}_{W_2}^T(\nu_{[0,s]},\nu_{[0,s]}')^2\dif s.
        \end{align*}
    Let $\Psi_T^k$ denote the $k$-fold composition of $\Psi_T$ with itself. Iterating the inequality above $k$-times yields
		\begin{align*}
			\mathsf{d}_{W_2}^T(\Psi_T^k(\nu),\Psi_T^k(\nu'))^2   &\leq  (c_T)^k \int_0^T  \frac{s^k}{(k-1)!}\mathsf{d}_{W_2}^T(\nu_{[0,s]},\nu_{[0,s]}')^2\dif s \\
			&\leq \frac{(c_T)^k}{k!}\mathsf{d}_{W_2}^T(\nu,\nu')^2\dif s.
		\end{align*}
In particular, for large enough $k$, $\Psi_T^k$ is a contraction. Because $(\cal{C}([0,T],\cal{P}_2(\r^{d_x})), \mathsf{d}_{W_2}^T)$ is a complete metric space~\citep[Proposition 17.15]{Sutherland2009}, the Banach--Caccioppoli fixed-point theorem~\citep[Theorem 5.1]{Kreyszig1978} then implies that $\Psi_T$ has a unique fixed point in  $\cal{C}([0,T],\cal{P}_2(\r^{d_x}))$.
    } \hfill\BlackBox \\

	In the lemma below, for some $T>0$ we set $\r^{d_x}_T:=[0,T)\times\r^{d_x}$ and we write  $\cal{L}^\infty_\textup{loc}(\r^{d_x}_T)$ for the set of real locally bounded functions on $\r^{d_x}_T$. We let $\cal{H}^{j,k}(\r^{d_x}_T)$ denote the space of real functions on $\r^{d_x}_T$ for which all components of the weak derivatives $\nabla_t^m  \nabla_x^n q$ exist and belong to $\cal{L}^\infty(\r^{d_x}_T)$ for all $m\leq j, n\leq k$. Here, $\nabla^i$ denotes the $i$-th fold outer product of $\nabla$ with itself. We write that such functions are in $\cal{H}^{j,k}_\textup{loc}(\r^{d_x}_T)$ if those weak derivatives only  belong to $\cal{L}^\infty_\textup{loc}(\r^{d_x}_T)$. In the proof below, we follow an argument in \citet[Lemma C.4]{Fan2023}, see also \citet[Theorem 5.1]{JKO1998}, \citet[Lemma 10.7]{Mei2018}.
	
	\begin{lemma} \label{lemma:solregularity}
		$\textup{Law}(X_t)$ has a Lebesgue density in $\cal{C}^{1,2}(\r^{d_x}_T,\r^+)$ and $\theta_t\in \cal{C}^1([0,T),\r^{d_\theta})$.
	\end{lemma}
	
	\noindent {\bf Proof}{.
		Set $q_t:=\textup{Law}(X_t)$. Since $(t,x)\mapsto \nabla_x \log(\rho_{\theta_t}(x))$ is locally integrable in $\r^{d_x}_T$ due Assumption \ref{ass:model} and the fact that $t\mapsto \theta_t$ is continuous, \citet[Corollary 6.4.3]{Bogachev2022} shows that $(q_t)_{t\leq T}$ admits a continuous Lebesgue density, further satisfying
		\begin{equation} \label{eq:sollocallydiff}
			q_t \in \cal{H}_{\textup{loc}}^{0,1}(\r^{d_x}_T).
		\end{equation}
		We now exploit this first regularity estimate and improve on it with a bootstrap argument.
		Let $\varphi_{\sigma^2}$ denote the Gaussian density with mean 0 and variance $\sigma^2$. Let $\phi\in\cal{C}_c^\infty(\r^{d_x},\r)$. For fixed $t>0$ and $y\in\r^{d_x}$, It\^{o}'s lemma on $(s,X_s)\mapsto \phi(X_s)\varphi_{\sigma^2+t-s}(y-X_s)$ yields	
		\begin{align*}
			\int &\phi(x)\varphi_{\sigma^2}(y-x)q_t(\dif x) =  
			\int_0^t \int (\partial_s \varphi_{\sigma^2+t-s}(y-x)\phi(x) 
			+\iprod{\nabla_x\ell(\theta_s,x)}{\nabla_x \phi}\varphi_{\sigma^2+t-s}(y-x)  \\ -
			&\iprod{\nabla_x \ell(\theta_s,x)}{\nabla_x\varphi_{\sigma^2+t-s}(y-x)}\phi(x)+
			\Delta_x(\phi(x)\varphi_{\sigma^2+t-s}(y-x)))q_s(\dif x)\dif s. 
		\end{align*}
		Using the heat equation $\partial_s\varphi_{\sigma^2+t-s}(y-x)=-\Delta_x\varphi_{\sigma^2+t-s}(y-x)$ and integrating by parts,
		\begin{align*}
			\int \phi(x)\varphi_{\sigma^2}(y-x)q_t(\dif x) 
			&= \int_0^t \int  (\iprod{\nabla_x \ell(\theta_s,x)}{\nabla_x \phi(x)}+\Delta_x \phi(x))q_s(x)\varphi_{\sigma^2+t-s}(y-x) \\
			&+  \iprod{(-\nabla_x \ell(\theta_s,x)\phi(x)+2\nabla_x\phi(x))q_s(x)}{\nabla_x\varphi_{\sigma^2+t-s}(y-x)}  \dif x\dif s.
		\end{align*}	
		Let $\sigma^2\rightarrow 0$. Using the weak differentiability of $q_t$ ensured by \eqref{eq:sollocallydiff}, we deduce 
		\begin{align} \label{eq:regularityingredient}
			\phi(y)q_t(y) 
			= \int_0^t \xi_{1,s}*\varphi_{t-s}(y) 
			+ \xi_{2,s}*\varphi_{t-s}(y)\dif s \quad \forall y\in\r^{d_x},
		\end{align}	
		where $*$ denotes the convolution operator, and we defined $\xi_{1,s}(x):=(\iprod{\nabla_x \ell(\theta_s,x)}{\nabla_x \phi(x)}-\Delta_x \phi(x)) q_s(x)$ and $\xi_{2,s}(x):=\nabla_x\cdot(q_s(x)(\nabla_x \ell(\theta_s,x)\phi(x)-2\nabla_x\phi(x)))$. The next key ingredient is the following implication in \citet[Chapter 4, (3.1)]{Ladyzhenskaia1988}:
		\begin{align} \label{eq:regularitybootstrap}
			\xi_t \in \cal{H}^{j,k}(\r^{d_x}_T) \text{ for } 2j+k\leq 2m \Rightarrow \int_0^t \xi_s *  \varphi_{t-s}(y) \dif s\in \cal{H}^{j,k}(\r^{d_x}_T) \text{ for } 2j+k\leq 2m+2
		\end{align}
		Since $q_t\in \cal{H}^{0,1}_{\textup{loc}}(\r^{d_x}_T)$, $\phi\in\cal{C}_c^\infty(\r^{d_x},\r)$ and the elements of $(t,x)\mapsto \nabla_{x} \ell(\theta_t,x)$ and $(t,x)\mapsto \nabla^2_x \ell(\theta_t,x)$ belong to $\cal{L}^\infty_{\textup{loc}}(\r^{d_x}_T)$ by Assumption \ref{ass:model}(ii), the equations \eqref{eq:regularitybootstrap} and \eqref{eq:regularityingredient} show the chain of implications $q_t \in \cal{H}^{0,1}_{\textup{loc}}(\r^{d_x}_T)\Rightarrow (\phi q_t)\in \cal{H}^{0,2}(\r^{d_x}_T)\Rightarrow  q_t \in \cal{H}^{0,2}_{\textup{loc}}(\r^{d_x}_T) \Rightarrow (\phi q_t)\in\cal{H}^{0,3}(\r^{d_x}_T)\Rightarrow  q_t \in \cal{H}^{0,3}_{\textup{loc}}(\r^{d_x}_T)$. 
		Next, since we now proved $q_t\in \cal{H}^{0,3}_{\textup{loc}}(\r^{d_x}_T)$, by analogous arguments  \eqref{eq:regularitybootstrap} also yields $(\phi q_t)\in \cal{H}^{1,3}(\r^{d_x}_T) \Rightarrow  q_t \in \cal{H}^{1,3}_{\textup{loc}}(\r^{d_x}_T)\Rightarrow(\phi q_t)\in \cal{H}^{2,3}(\r^{d_x}_T) \Rightarrow  q_t \in \cal{H}^{2,3}_{\textup{loc}}(\r^{d_x}_T)$. Finally, by the Sobolev embedding Theorem \citep[Theorem 4.12]{Adams2003}, $ q_t \in \cal{H}^{2,3}_{\textup{loc}}(\r^{d_x}_T)\Rightarrow q_t\in\cal{C}^{1,2}(\r^{d_x}_T)$. Since $\ell$ is continuously differentiable by Assumption \ref{ass:gradLip} and $q_t\in\cal{C}^{1,2}(\r^{d_x}_T)$, we also have $\theta_t \in \cal{C}^1([0,T),\r^{d_\theta})$, and the claim follows.	
	} \hfill\BlackBox \\

    Following the, by now standard, argument of \citet{Sznitman1991}, as $T$ is arbitrary and for any $T'<T$ the projection of the solution on $[0,T]$ onto $[0,T']$ coincides with the solution obtained by working directly on $[0,T']$, there exists an unique extension on $[0,\infty)$. The regularity properties of the solutions given in Lemma \ref{lemma:solregularity} extends to $[0,\infty)$. In fact, the marginals of the solution on $[0,\infty)$ have to agree to the solution on $[0,T)$ for any $T>0$; such regularity is guaranteed by Lemma \ref{lemma:solregularity}. Hence, if the regularity failed at some $T'>0$, we could just take $T=2T'$, producing a contradiction.
 
	We proved that \eqref{eq:sde1} has an unique strong solution on all $[0,\infty)$ with the regularity that Assumption \ref{ass:solutions_existence} asks a classical solution to \eqref{eq:pde} to have. The next lemma closes the circle.

	\begin{lemma} \label{lemma:sdepdesol}
		$(\theta_t,\textup{Law}(X_t))_{t\geq 0}$ is a classical solution to \eqref{eq:pde}.
	\end{lemma}
	\noindent {\bf Proof}{.
		Let $ \phi $ in $ \cal{C}_c^\infty(\r^{d_x},\r)$. It\^{o}'s lemma shows
		\begin{align*}
			\phi(X_t) = \phi(X_0) + \int_0^t \bigg(\iprod{\nabla_x\ell(\theta_s,X_s)}{\nabla_x\phi(X_s)} + \Delta_x\phi(X_s)\bigg) \dif s + \int_0^t \iprod{\nabla_x\phi(X_s)}{\dif W_s}. 
		\end{align*}
		Consider the system
		\begin{align*}
			\dot{\theta}_t = \int \nabla_\theta\ell(\theta_t,x)q_t(\dif x), \quad M_t:= \phi(X_t)-\int_0^t \bigg(\iprod{\nabla_x\ell(\theta_s,X_s)}{\nabla_x\phi(X_s)} + \Delta_x\phi(X_s)\bigg) \dif s
		\end{align*}
		with $q_t=\text{Law}(X_t)$. Comparing with the expression for $\phi(X_t)$ above we notice that $M_t$ is a martingale with respect to the natural filtration generated by $(\theta_t,X_t)_{t\geq 0}$, as it corresponds to the It\^{o} integral of an adapted, square integrable process against Brownian motion. Taking expectations and the time derivative in $M_t$ above shows that
		\begin{align*}
			\frac{\dif}{\dif t}\int \phi(x)q_t(\dif x)  &= -\int \iprod{\nabla_x \ell(\theta_t,x)}{\nabla_x \phi(x)}  q_t(\dif x) + \int \Delta_x \phi(x)q_t(\dif x) \quad\forall \phi \in \cal{C}_c^\infty(\r^{d_x},\r).
		\end{align*}
		In particular, this shows that $(\theta_t,\textup{Law}(X_t))_{t\geq 0}$ is a weak solution to \eqref{eq:pde}. Thanks to the regularity provided by Lemma \ref{lemma:solregularity}, we can integrate by parts the above to obtain
		\begin{equation*}
			\frac{\dif}{\dif t}\int \phi(x)q_t(x)\dif x  = \int \phi(x)(\Delta_x q_t(x) - \iprod{\nabla_x \ell(\theta_t,x)}{\nabla_x q_t(x)})\dif x;
		\end{equation*}
		for all $\phi\in\cal{C}_c^\infty(\r^{d_x},\r)$. For each $x\in\r^{d_x}$, we can consider a sequence of bump functions in $\cal{C}_c^\infty(\r^{d_x},\r)$ concentrating in the point, showing that \eqref{eq:pde} holds pointwise and that $(\theta_t,\textup{Law}(X_t))_{t\geq 0}$ is a classical solution.		
	} \hfill\BlackBox \\

	\subsection{Proof of Lemma~\protect{\ref{lemma:iidapprox}}} \label{app:iidapprox}
	Recall that the particles $(X_t^1,\dots,X^N_t)$ are i.i.d.~with distribution $q_t$, and that the particles $(X_*^1,\dots,X^N_*)$ are i.i.d.~with distribution $\pi_{\theta_*}$. Consider the coupling of these random vectors such that each $X_*^i$ is optimally coupled (in the sense of the Wasserstein-2 distance) with $X_{t}^i$. In this case, we write
	\begin{align*}
		\mathsf{d}((\theta_{t},Q^N_{t})),(\theta_*,Q_*^N))^2&\leq \norm{\theta_{t}-\theta_*}^2+\frac{1}{N}\sum_{n=1}^N \Ebb{\norm{X_{t}^n-X_*^n}^2} \\
		&= \norm{\theta_{t}-\theta_*}^2+\Ebb{\norm{X_{t}^1-X_*^1}^2} = \mathsf{d}((\theta_{t},q_{t})),(\theta_*,\pi_{\theta_*}))^2.
	\end{align*}
	The result follows via Theorem \ref{thm:dflowconvergencehp} and Proposition \ref{prop:sdesolutions}. \hfill\BlackBox \\
	
	\subsection{Proof of Lemma~\protect{\ref{lemma:propchaos}}} \label{app:propchaos}
	\noindent
	\begin{proposition} \label{prop:uniformmomentboundflow}
		Assume Assumptions~\ref{ass:strongconcave}--\ref{ass:gradLip} and let $(\theta_\dagger,x_\dagger)$ denote $\ell$'s unique maximizer. Then, we have the following uniform-in-time second moment bound  
		$$
		\norm{\theta_t}^{2}+\Ebb{\norm{X_t}^{2}} \leq 2\left(\norm{\theta_\dagger}^2+\norm{x_\dagger}^{2}+\norm{\theta_0-\theta_\dagger}^2+\Ebb{\norm{X_0-x_\dagger}^{2}}+\frac{2d_x}{\lambda}\right)\quad\forall t\geq0.
		$$
		If $(\theta_\dagger,x_\dagger)=(0,0)$, the above also holds without the factor of 2 in the right hand side.
	\end{proposition}
	
	\noindent {\bf Proof}{.
		Because,  
		$$f(z):=\norm{z}^{2}\Rightarrow\nabla f(z)=2z,\quad \nabla^2 f(z)=2I,$$
		by setting 
		$$\xi_t:=\norm{X_t-x_\dagger}^{2}+\norm{\theta_t-\theta_\dagger}^{2}=\norm{(\theta_t,X_t)-(\theta_\dagger,x_\dagger)}^{2},$$
		applying It\^{o}'s formula  \cite[Theorem 4.2.1]{Oksendal2013} and noting that $X_t$ has law $q_t$, we find 
		\begin{align}
			\dif \xi_t=2\left[\iprod{(\theta_t-\theta_\dagger,X_t-x_\dagger)}{(\Ebb{\nabla_\theta\ell(\theta_t,X_t)},\nabla_x\ell(\theta_t,X_t))}+d_x\right]\dif t+2^{3/2}\iprod{X_t-x_\dagger}{\dif W_t}.\label{eq:dnsa8dnysuadnau}
		\end{align}   
		However,
		\begin{align}&\iprod{(\theta_t-\theta_\dagger,X_t-x_\dagger)}{(\Ebb{\nabla_\theta\ell(\theta_t,X_t)},\nabla_x\ell(\theta_t,X_t))}\label{eq:nusd9adnsuadua}\\
			&\qquad\qquad\qquad\qquad\qquad\qquad\qquad\qquad=\iprod{\theta_t-\theta_\dagger}{\Ebb{\nabla_\theta\ell(\theta_t,X_t)}}+\iprod{X_t-x_\dagger}{\nabla_x\ell(\theta_t,X_t)}\nonumber\\
			&\qquad\qquad\qquad\qquad\qquad\qquad\qquad\qquad=\Ebb{\iprod{\theta_t-\theta_\dagger}{\nabla_\theta\ell(\theta_t,X_t)}}+\iprod{X_t-x_\dagger}{\nabla_x\ell(\theta_t,X_t)},\nonumber\end{align}
		and, because $\ell$ is strongly log-concave by Assumption \ref{ass:strongconcave} and $(\theta_\dagger,x_\dagger)$ maximizes it, 
		\begin{align*} 
			\iprod{\nabla\ell(\theta_t,X_t)}{(\theta_t,X_t)-(\theta_\dagger,x_\dagger)}\leq \ell(\theta_t,X_t)-\ell(\theta_\dagger,x_\dagger)-\lambda\xi_t/2\leq -\lambda\xi_t/2.
		\end{align*}
		(e.g. see pp. 63-64 in \citealp{Nesterov2003}). Taking expectations of~\eqref{eq:nusd9adnsuadua} and applying the above, 
		\begin{align*}&\Ebb{\iprod{(\theta_t-\theta_\dagger,X_t-x_\dagger)}{(\Ebb{\nabla_\theta\ell(\theta_t,X_t)},\nabla_x\ell(\theta_t,X_t))}}
			\\
			&=\Ebb{\iprod{(\theta_t-\theta_\dagger,X_t-x_\dagger)}{\nabla\ell(\theta_t,X_t)}}
			\leq-\lambda\Ebb{\xi_t}/2.\end{align*}
		In turn, taking expectations of~\eqref{eq:dnsa8dnysuadnau} yields $\dif \Ebb{\xi_t} \leq[-\lambda\Ebb{\xi_t}+2d_x]\dif t\enskip\forall t\geq0$ and
		$$\Ebb{\xi_t}\leq e^{-\lambda t}\left(\Ebb{\xi_0}-\frac{2d_x}{\lambda}\right)+\frac{2d_x}{\lambda}\leq \max\bigg(\Ebb{\xi_0},\frac{2d_x}{\lambda}\bigg) \leq \Ebb{\xi_0} + \frac{2d_x}{\lambda}\quad\forall t\geq0,$$
		where we used $e^{-\lambda t}\leq 1$. Now the $C_p$ inequality, $\norm{\theta_t}^{2}\leq 2(\norm{\theta_\dagger}^2+\norm{\theta_t-\theta_\dagger}^2)$ and similarly for $\Ebb{\norm{X_t}^{2}}$, gives the desired bound.
	} \hfill\BlackBox \\
	
	\noindent {\bf Proof of Lemma~\protect{\ref{lemma:propchaos}}.}{
		To prove the desired bound, we start with It\^{o}'s formula to write
		\begin{align*}
			\dif \norm{\Theta^N_t-\theta_t}^2&=2\iprod{\Theta^N_t-\theta_t}{\frac{1}{N}\sum_{n=1}^N\nabla_\theta \ell(\Theta_t^N,\bar{X}_{t}^n)-\int\nabla_\theta \ell(\theta_t,x)q_t(\dif x)}\dif t,\\
			\dif \norm{\bar{X}_{t}^n-X^n_t}^2&=2\iprod{\bar{X}_{t}^n-X_t^n}{\nabla_x\ell(\Theta_t^N,\bar{X}_{t}^n)-\nabla_x\ell(\theta_t,X_t^n)}\dif t.
		\end{align*}
		Setting $\bar{\xi}_{t}^N:=N^{-1}\sum_{n=1}^N\norm{\bar{X}_{t}^n-X^n_t}^2+\norm{\Theta_t^N-\theta_t}^2$, averaging the second equation over $n$, and adding to the first, we find that
		\begin{align*}
			\dif \bar{\xi}_{t}^N=&2\iprod{\Theta^N_t-\theta_t}{\frac{1}{N}\sum_{n=1}^N\nabla_\theta \ell(\Theta_t^N,\bar{X}_{t}^n)-\int\nabla_\theta \ell(\theta_t,x)q_t(\dif x) }\dif t\nonumber
			\\&+\frac{2}{N}\sum_{n=1}^N\iprod{\bar{X}_{t}^n-X_t^n}{\nabla_x\ell(\Theta_t^N,\bar{X}_{t}^n)-\nabla_x\ell(\theta_t,X_t^n)}\dif t.
		\end{align*}
		Because $X_t^1,\dots,X_t^N$ all have law $q_t$,
		\begin{align*}
			\dif \bar{\xi}_{t}^N=\frac{2}{N}\sum_{n=1}^N\Big[&\iprod{\Theta^N_t-\theta_t}{\nabla_\theta \ell(\Theta_t^N,\bar{X}_{t}^n)-\nabla_\theta \ell(\theta_t,X_t^n)} \\ 
			&+\iprod{\bar{X}_{t}^n-X_t^n}{\nabla_x\ell(\Theta_t^N,\bar{X}_{t}^n)-\nabla_x\ell(\theta_t,X_t^n)}\Big]\dif t+2G_t^N\dif t  
		\end{align*}
		where
		\begin{align*}
			G_t^N:&=\iprod{\Theta^N_t-\theta_t}{\frac{1}{N}\sum_{n=1}^N\nabla_\theta \ell(\theta_t,X_t^n)-\int\nabla_\theta \ell(\theta_t,x)q_t(\dif x)}\\
			&=\frac{1}{N}\iprod{\Theta^N_t-\theta_t}{\sum_{n=1}^N\left[\nabla_\theta \ell(\theta_t,X_t^n)-\Ebb{\nabla_\theta \ell(\theta_t,X_t^n)}\right]}.
		\end{align*}
		It then follows from Assumption~\ref{ass:strongconcave} that
		\begin{align}
			\dif \bar{\xi}_{t}^N\leq [-2\lambda \bar{\xi}_{t}^N +2G_t^N]\dif t\enskip\forall t\geq0,\quad\Rightarrow\quad \frac{\dif \Ebb{\bar{\xi}_{t}^N}}{\dif t}\leq -2\lambda \Ebb{\bar{\xi}_{t}^N} +2\Ebb{G_t^N}\enskip \forall t\geq0.\label{eq:propchaosxiNineq}
		\end{align}
		As we show below, 
		\begin{equation}
			\label{eq:EGtbound}\mmag{\Ebb{G_t^N}}\leq L\sqrt{\frac{2\Ebb{\bar{\xi}_{t}^N} }{N}\left(B_0+\frac{d_x}{\lambda}\right)}.
		\end{equation}
		Because $\frac{\dif}{\dif t}\Ebb{\bar{\xi}_{t}^N}^{1/2}=2^{-1} \Ebb{\bar{\xi}_{t}^N}^{-1/2} \frac{\dif}{\dif t}\Ebb{\bar{\xi}_{t}^N}$, (\ref{eq:propchaosxiNineq},\ref{eq:EGtbound}) imply that
		$$\frac{\dif}{\dif t}\Ebb{\bar{\xi}_{t}^N}^{1/2} \leq -\lambda\Ebb{\bar{\xi}_{t}^N}^{1/2} + L\sqrt{\frac{2}{N}\left(B_0+\frac{d_x}{\lambda}\right)}.$$
		Applying Gr\"onwall's inequality, we obtain that
		\begin{align*}
			\Ebb{\bar{\xi}_{t}^N}^{1/2} &\leq e^{-\lambda t}\sqrt{\bar{\xi}_{0}^N}+ \frac{(1-e^{-\lambda t})L}{\lambda}\sqrt{\frac{2}{N}\left(B_0+\frac{d_x}{\lambda}\right)},
		\end{align*}
		and the result follows because  $\bar{\xi}_0^N=0$ by construction. 
	} \hfill\BlackBox \\
	
	\noindent {\bf Proof of~\protect{\eqref{eq:EGtbound}}.}{
		Applying the Cauchy--Schwarz inequality, we obtain 
		\begin{align*}
			\mmag{\Ebb{G_t^N}}^2
			&\leq \frac{1}{N^2}\Ebb{\norm{\Theta^N_t-\theta_t}^2}\Ebb{\norm{\sum_{n=1}^N\left[\nabla_\theta \ell(\theta_t,X_t^n)-\Ebb{\nabla_\theta \ell(\theta_t,X_t^n)}\right]}^2} \\
			&\leq \frac{\Ebb{\bar{\xi}_{t}^N}}{N^2}\sum_{n=1}^N\Ebb{\norm{\nabla_\theta \ell(\theta_t,X_t^n)-\Ebb{\nabla_\theta \ell(\theta_t,X_t^n)}}^2}=:\frac{\Ebb{\bar{\xi}_{t}^N}}{N^2}\sum_{n=1}^Nc_t^n,
		\end{align*}
		where the last inequality follows from the independence of $X_t^1,\dots,X_t^N$. Let $X_t'$ denote a random variable independent of $(X_t^1,\dots,X_t^n)$ with law $q_t$. Jensen's inequality and  Lipschitz continuity of $\nabla_\theta \ell$ (Assumption~\ref{ass:gradLip}) imply
		\begin{align*}
			c_t^n&= \Ebb{\norm{\nabla_\theta \ell(\theta_t,X_t^n)-\Ebb{\nabla_\theta \ell(\theta_t,X_t^n)}}^2} \leq \Ebb{\norm{\nabla_\theta \ell(\theta_t,X_t^n)-\nabla_\theta \ell(\theta_t,X_t')}^2} \\
			& \leq L^2 \Ebb{\norm{X_t^n-X_t'}^2} \leq 2L^2 \Ebb{\norm{X_t'}^2}.
		\end{align*}
		Combining the above two and  Proposition~\ref{prop:uniformmomentboundflow}, we then obtain~\eqref{eq:EGtbound}.
	} \hfill\BlackBox \\

	\subsection{Proof of Lemma~\protect{\ref{lemma:disc}}} \label{app:discr} 
	The following proof extends the arguments in \citet[Lemma S2]{Durmus2019} or \citet[Chapter 4.1]{Chewi2023}. We introduce the linear interpolation of the Euler--Maruyama discretization of \eqref{eq:spacediscrsystems2}:
	\begin{equation*} \label{eq:timediscrsys2}
		\begin{aligned} 
			\Tilde{X}_t^{n} =& \Tilde{X}_{kh}^{n} + (t-kh)\nabla_x\ell(\Tilde{X}_{kh}^{n},\Tilde{\Theta}_{kh}^{N}) + \sqrt{2}(W^n_t-W^n_{kh}), \quad \forall n\in[N]; \\  
			\Tilde{\Theta}_t^{N} =& \Tilde{\Theta}_{kh}^{N}  + (t-kh)\frac{1}{N}\sum_{n=1}^N\nabla_\theta\ell(\Tilde{X}_{kh}^{n},\Tilde{\Theta}_{kh}^{N});
		\end{aligned}
	\end{equation*}
	for all $t\in [kh,(k+1)h)$ and $k\in\n$. Let $\Tilde{Q}_t^h:=N^{-1}\sum_{i=1}\delta_{\Tilde{X}_t^{n}}$. We notice that $(\Tilde{\Theta}_{Kh}^{N},\Tilde{Q}_{Kh}^h)$ coincides in distribution with $(\Theta^{N,h}_{K},Q^{N,h}_{K})$, hence we just need to derive the bound 
	\begin{equation*}
		\mathsf{d}((\Tilde{\Theta}_{Kh}^{N},\Tilde{Q}_{Kh}^h),(\Theta^N_{Kh},\bar{Q}^N_{Kh})) \leq \sqrt{h}A_{0,h} \quad \forall K\in\mathbb{N}.
	\end{equation*}
	For all $n\in[N]$, $k\in\n$ and $h>0$ we compute directly from the defining equations,
	\begin{align*}
		&\norm{\bar{X}_{(k+1)h}^n-\Tilde{X}_{(k+1)h}^{n}}^2 = 
		\norm{\int_{kh}^{(k+1)h}[\nabla_x \ell(\bar{X}_s^n,\Theta_s^N)-\nabla_x \ell(\Tilde{X}_{kh}^{n},\Tilde{\Theta}_{kh}^{N}) ]\dif s}^2 \\
		&- 2h\iprod{\bar{X}_{kh}^n-\Tilde{X}_{kh}^{n}}{\nabla_x \ell(\bar{X}_{kh}^n,\Theta_{kh}^N)-\nabla_x \ell(\Tilde{X}_{kh}^{n},\Tilde{\Theta}_{kh}^{N})} \\
		&- 2\int_{kh}^{(k+1)h} \iprod{\bar{X}_{kh}^n-\Tilde{X}_{kh}^{n}}{\nabla_x \ell(\bar{X}_{s}^n,\Theta_{s}^N)-\nabla_x \ell(\bar{X}_{kh}^n,\Theta_{kh}^N)}\dif s + \norm{\bar{X}_{kh}^n-\Tilde{X}_{kh}^{n}}^2,
	\end{align*}
	\begin{align*}
		&\norm{\Theta_{(k+1)h}^N-\Tilde{\Theta}_{(k+1)h}^{N}}^2 = \frac{1}{N}\sum_{n=1}^N\norm{2\int_{kh}^{(k+1)h}[\nabla_\theta\ell(\bar{X}_{s}^n,\Theta_{s}^N)-\nabla_\theta\ell(\Tilde{X}_{kh}^{n},\Tilde{\Theta}_{kh}^{N})]\dif s}^2 \\ 
		&- \frac{2h}{N}\sum_{n=1}^N\iprod{\Theta_{kh}^N-\Tilde{\Theta}_{kh}^{N}}{\nabla_\theta\ell(\bar{X}_{kh}^n,\Theta_{kh}^N)-\nabla_\theta\ell(\Tilde{X}_{kh}^{n},\Tilde{\Theta}_{kh}^{N})}\\
		& - \frac{2}{N}\sum_{n=1}^N\int_{kh}^{(k+1)h} \iprod{\Theta_{kh}^N-\Tilde{\Theta}_{kh}^{N}}{[\nabla_\theta\ell(\bar{X}_{s}^n,\Theta_{s}^N)-\nabla_\theta\ell(\bar{X}_{kh}^n,\Theta_{kh}^N)]}\dif s + \norm{\Theta_{kh}^N-\Tilde{\Theta}_{kh}^{N}}^2.
	\end{align*}
	Averaging the $N$ equations for $\norm{\bar{X}_{(k+1)h}^n-\Tilde{X}_{(k+1)h}^{n}}^2$, adding the one for $\norm{\Theta_{(k+1)h}^N-\Tilde{\Theta}_{(k+1)h}^{N}}^2$ we obtain, with $\bar{Y}_t^n:=(\bar{X}_t^n,\Theta_t^N)$, $\Tilde{Y}^{n}_t:=(\Tilde{X}_{t}^{n},\Tilde{\Theta}_t^{N})$ and $\xi_t:=N^{-1}\sum_{n=1}^N\norm{\Tilde{Y}_t^{n}-\bar{Y}_t^n}^2$,
	\begin{align}
		\xi_{(k+1)h}=& 
		\xi_{kh} + 
		\frac{1}{N}\sum_{n=1}^N\norm{\int_{kh}^{(k+1)h}[\nabla \ell(\bar{Y}_s^n)-\nabla \ell(\Tilde{Y}^{n}_{kh})]\dif s }^2 \nonumber \\
		&-\frac{2h}{N}\sum_{n=1}^N\iprod{\bar{Y}_{kh}^n-\Tilde{Y}_{kh}^{n}}{\nabla \ell (\bar{Y}_{kh}^n)-\nabla \ell(\Tilde{Y}_{kh}^{n})} \nonumber \\ 
		&-\frac{2}{N}\sum_{n=1}^N\int_{kh}^{(k+1)h}\iprod{\bar{Y}_{kh}^n-\Tilde{Y}_{kh}^{n}}{\nabla \ell(\bar{Y}_{s}^n)-\nabla \ell(\bar{Y}_{kh}^n)}\dif s. \label{eq:discrwithxi}
	\end{align}
	Now, adding and subtracting $\nabla \ell(\bar{Y}_{kh}^n)$, applying Jensen's inequality, expanding the resulting integrand and further applying Young's inequalities,
	\begin{align} 
		\norm{\int_{kh}^{(k+1)h}[\nabla \ell(\bar{Y}_s^n)-\nabla \ell(\Tilde{Y}^{n}_{kh})]\dif s}^2
		&\leq 2h^2\norm{\nabla \ell(\bar{Y}_{kh}^n)-\nabla \ell(\Tilde{Y}_{kh}^{n})}^2  \nonumber \\ \label{eq:discrwithxibound1}
		&
		+ 2h\int_{kh}^{(k+1)h} \norm{\nabla \ell(\bar{Y}_s^n)-\nabla \ell(\bar{Y}_{kh}^n)}^2 \dif s.
	\end{align}
	Furthermore, under Assumptions \ref{ass:strongconcave} and \ref{ass:gradLip} we have, by \citet[Theorem 2.1.12]{Nesterov2003}, there holds the co-coercivity property
	\begin{equation*}
		\iprod{\nabla \ell(y')-\nabla\ell(y)}{y'-y} \geq \frac{\iota}{2} \norm{y-y'}^2 + \frac{\norm{\nabla \ell(y')-\nabla \ell(y)}^2}{\lambda+L},
	\end{equation*}
	where $\iota:=2\lambda L/(\lambda +L)$. It follows that, whenever $h<1/(\lambda+L)$, we have 
	\begin{align} 
		\frac{1}{N}\sum_{n=1}^N\left[2h^2\norm{\nabla \ell(\bar{Y}_{kh}^n)-\nabla \ell(\Tilde{Y}_{kh}^{n})}^2\right.&\left. - 2h\iprod{\bar{Y}_{kh}^n-\Tilde{Y}_{kh}^{n}}{\nabla \ell(\bar{Y}_{kh}^n)-\nabla \ell(\Tilde{Y}_{kh}^{n})}\right] \nonumber \\
		&\leq -\frac{h\iota}{N}\sum_{n=1}^N\norm{\bar{Y}_{kh}^n-\Tilde{Y}_{kh}^{n}}^2 = -h\iota \xi_{kh}\label{eq:discrwithxibound2}
	\end{align}
	Combining (\ref{eq:discrwithxi},\ref{eq:discrwithxibound1},\ref{eq:discrwithxibound2}), we obtain
	\begin{align*}
		\xi_{(k+1)h} \leq& (1-\iota h)\xi_{kh} -\frac{2}{N}\sum_{n=1}^N\left[\int_{kh}^{(k+1)h}\iprod{\bar{Y}_{kh}^n-\Tilde{Y}_{kh}^{n}}{\nabla \ell(\bar{Y}_{s}^n)-\nabla \ell(\bar{Y}_{kh}^n)}\dif s\right. \nonumber\\
		&\left.+2h\int_{kh}^{(k+1)h} \norm{\nabla \ell(\bar{Y}_s^n)-\nabla \ell(\bar{Y}_{kh}^n)}^2\dif s\right]. 
	\end{align*}
	To deal with the remaining terms, fix any $\epsilon>0$. Applying Young's inequality, we obtain 
	\begin{align*}
		\frac{1}{N}\sum_{n=1}^N&\mmag{\iprod{\bar{Y}_{kh}^n-\Tilde{Y}_{kh}^{n}}{\nabla \ell(\bar{Y}_{s}^n)-\nabla \ell(\bar{Y}_{kh}^n)}} \\
		&\leq \frac{1}{N}\sum_{n=1}^N\left[\frac{\epsilon}{2}\norm{\bar{Y}_{kh}^n-\Tilde{Y}_{kh}^{n}}^2 + \frac{1}{2\epsilon}\norm{\nabla \ell(\bar{Y}_{s}^n)-\nabla \ell(\bar{Y}_{kh}^n)}^2\right] \\
		&=\frac{\epsilon}{2} \xi_{kh}+\frac{1}{2\epsilon N}\sum_{n=1}^N \norm{\nabla \ell(\bar{Y}_{s}^n)-\nabla \ell(\bar{Y}_{kh}^n)}^2.
	\end{align*}
	Putting the above two together, we obtain for any $\epsilon>0$ and $k\in\n$,
	\begin{align} \label{eq:discrxibound}
		\xi_{(k+1)h} \leq (1-\iota h+h\epsilon/2)\xi_{kh} + \left(2h+\frac{1}{2\epsilon}\right) \int_{kh}^{(k+1)h} \frac{1}{N}\sum_{n=1}^N\norm{\nabla \ell(\bar{Y}_{s}^n)-\nabla \ell(\bar{Y}_{kh}^n)}^2\dif s.
	\end{align}
	We will show below that if $h<1/(\lambda+L)$ then, for all $s\in[kh,(k+1)h)$,
	\begin{align} \label{eq:discrnablabound}
		\frac{1}{N}\sum_{n=1}^N\Ebb{\norm{\nabla \ell(\bar{Y}_{s}^n)-\nabla \ell(\bar{Y}_{kh}^n)}^2} \leq 220hL^2(L^2h(B_0+d_x/\lambda + d_x)).
	\end{align}
	Hence, taking expectations in \eqref{eq:discrxibound}, choosing $\epsilon=\iota$ using the bound above and the fact that $\xi_0=0$ we obtain 
	\begin{align*}
		\Ebb{\xi_{(k+1)h}} &\leq (2h+2/\iota)220h^2(L^2h(B_0+d_x/\lambda+ d_x)\sum_{j=1}^k (1-\iota h/2)^j  \\
		&\leq \frac{4h+4/\iota}{\iota}220hL^2(L^2h(B_0+d_x/\lambda+ d_x).
	\end{align*} \hfill\BlackBox \\
	
	\noindent {\bf Proof of \eqref{eq:discrnablabound}.}{
		Let $\{\mathcal{F}_{t};t\geq 0\}$ be the filtration generated by $(\theta_t,\bar{X}_t^n)_{n=1}^N$ and denote $\mathbb{E}_{kh}[\cdot]$ expectation conditional on $\mathcal{F}_{kh}$. We first show that whenever $h<1/(\lambda+L)$
		\begin{equation} \label{eq:discrnablaboundaux}
			\frac{1}{N} \sum_{n=1}^N\mathbb{E}_{kh}\left[\norm{\nabla \ell(\bar{Y}_{s}^n)-\nabla \ell(\bar{Y}_{kh}^n)}^2\right] \leq \frac{1}{N}\sum_{n=1}^N 220hL^2(L^2h\norm{\bar{Y}^n_{kh}}^2 + d_x),
		\end{equation}
		after which the result follows from taking expectations and applying Proposition \ref{prop:uniformmomentboundips}. Since
		\begin{align*}
			\mathbb{E}_{kh}&\left[\norm{\bar{X}_{s}^n-\bar{X}_{kh}^n}^2\right] =  \mathbb{E}_{kh}\left[ \norm{\int_{kh}^s \nabla_x \ell(\Theta_r^N,\bar{X}^n_r)\dif r+\sqrt{2}(W^n_s-W^n_{kh})}^2\right] \\
			&\leq 2(s-kh)\int_{kh}^s \mathbb{E}_{kh}\left[\norm{\nabla_x \ell(\Theta_r^N,\bar{X}^n_r)}^2\right]\dif r + 4d_x(s-kh)
		\end{align*}
		for all $n\in[N]$, and
		\begin{align*}
			\norm{\Theta_{s}^N-\Theta_{kh}^N}^2=   \norm{\frac{1}{N}\sum_{n=1}^N\int_{kh}^s \nabla_\theta \ell(\Theta_r^N,\bar{X}_r^n)\dif r}^2 \leq 2(s-kh)\frac{1}{N}\sum_{n=1}^N\int_{kh}^s \norm{\nabla_\theta \ell(\Theta_r^N,\bar{X}_r^n)}^2 \dif r,
		\end{align*}
		adding up and since Assumption \ref{ass:gradLip} implies $\norm{\nabla\ell(\bar{Y}_r^n)}^2\leq L^2\norm{\bar{Y}_r^n}^2$, we obtain 
		\begin{align*}
			\frac{1}{N}& \sum_{n=1}^N  \mathbb{E}_{kh}\left[\norm{\bar{Y}^n_{s}-\bar{Y}^n_{kh}}^2\right] \leq 2hL^2\int_{kh}^s \frac{1}{N} \sum_{n=1}^N  \mathbb{E}_{kh}\left[\norm{\bar{Y}^n_{r}}^2\right]\dif r  + 4d_xh \\
			&\leq 4hL^2\int_{kh}^s \frac{1}{N} \sum_{n=1}^N  \mathbb{E}_{kh}\left[\norm{\bar{Y}^n_{r}-\bar{Y}^n_{kh}}^2\right]\dif r + 4h^2L^2  \mathbb{E}_{kh}\left[\norm{\bar{Y}^n_{kh}}^2\right] +  4d_xh
		\end{align*}
		and by Gr\"onwall's lemma
		\begin{align*}
			\frac{1}{N} \sum_{n=1}^N  \mathbb{E}_{kh}&\left[\norm{\bar{Y}^n_{s}-\bar{Y}^n_{kh}}^2\right] \leq \frac{1}{N} \sum_{n=1}^N\exp(4L^2h^2)(4L^2h^2\norm{\bar{Y}^n_{kh}}^2+4d_xh).
		\end{align*}
		If $h<1/(\lambda+L)$, $\exp(4L^2h^2)\leq 55$, and \eqref{eq:discrnablaboundaux} then follows by  Assumption \ref{ass:gradLip}.  } \hfill\BlackBox \\

	\begin{proposition}\label{prop:uniformmomentboundips}
		Assume Assumptions~\ref{ass:strongconcave}--\ref{ass:gradLip}. Then, for all $t\geq 0$,
		$$
		\frac{1}{N} \sum_{n=1}^N \Ebb{\norm{X_t^n}^{2}+\norm{\Theta_t^N}^{2}} \leq 2\left(\norm{\theta_0}^2+\Ebb{\norm{\bar{X}_0}^{2}}+\frac{d_x}{\lambda}\right)$$
	\end{proposition}
	
	\noindent {\bf Proof}{.
		Recalling that here $(\theta_\dagger,x_\dagger)=(0,0)$, this argument is very similar to that supporting Proposition~\ref{prop:uniformmomentboundflow}. For each $n\in[N]$,
		\begin{align*}
			\dif \norm{X_t^n}^2 &= - 2\iprod{X_t^n}{\nabla_x\ell(\Theta_t^N,X_t^n)}\dif t + \sqrt{2}\dif W_t^n + 2d_x \dif t \\
			\dif \norm{\Theta_t^N}^2 &= -\frac{2}{N}\sum_{n=1}^N \iprod{\Theta_t^N}{\nabla_\theta\ell(\Theta_t^N,X_t^n)}\dif t 
		\end{align*}
		now averaging the $N$ equations for $\norm{X_t^n}^2$ and then adding $\norm{\Theta_t^N}^2$, taking expectations and time derivatives, with $\xi_{t}^N=N^{-1}\sum_{n=1}^N\norm{X_t^n}^2+\norm{\Theta_t^N}^2$,
		\begin{align*}
			\frac{\dif}{\dif t}\Ebb{\xi_{t}^N} 
			= -2 \Ebb{\frac{1}{N}\sum_{n=1}^N\iprod{(X^n_t,\Theta_t^N)}{\nabla\ell(\Theta_t^N,X^n_t)}} + 2d_x 
			\leq -2\lambda \Ebb{\xi_{t}^N} + 2d_x,
		\end{align*}
		and Gr\"onwall's inequality provides the conclusion.
	} \hfill\BlackBox \\

\bibliography{pgd_theory}
 
\end{document}